\documentclass{article}

\usepackage[final]{staix_2026}         % camera-ready, Full Paper Track

\usepackage[utf8]{inputenc}
\usepackage[T1]{fontenc}
\usepackage{xurl}
\usepackage{hyperref}
\usepackage{url}
\usepackage{booktabs}
\usepackage{amsfonts}
\usepackage{amssymb}
\usepackage{amsmath}
\usepackage{microtype}
\usepackage{nicefrac}
\usepackage{xcolor}
\usepackage{graphicx}
\usepackage{multirow}
\usepackage{array}
\usepackage{xcolor}
\usepackage{subcaption}
\usepackage{amsthm}
\usepackage{bbm}
\usepackage{dsfont}
\theoremstyle{plain}
\usepackage{multirow}
\usepackage{booktabs}
\usepackage{tabularx}
\usepackage{array}

\usepackage{algorithm}   % for algorithm environment
\usepackage{algorithmic} % for algorithmic environment

\newtheorem{theorem}{Theorem}[section]

\newtheorem{assumption}{Assumption}[section]

\title{NeuroMAS: Multi-Agent Systems as Neural Networks with Joint Reinforcement Learning}

\author{
  Haoran Lu \\
  Department of Statistics \\
  University of Georgia \\
  Athens, GA 30602 \\
  \texttt{haoran.lu@uga.edu}
  \And
  Luyang Fang \\
  Department of Statistics \\
  University of Georgia \\
  Athens, GA 30602 \\
  \texttt{luyang.fang@uga.edu}
  \And
  Wenxuan Zhong \\
  Department of Statistics \\
  University of Georgia \\
  Athens, GA 30602 \\
  \texttt{wenxuan@uga.edu}
  \And
  Ping Ma\thanks{Corresponding author.}  \\
  Department of Statistics \\
  University of Georgia \\
  Athens, GA 30602 \\
  \texttt{pingma@uga.edu}
}

\begin{document}

\maketitle

\begin{abstract}
Multi-agent language systems are often built as hand-designed workflows, where agents are assigned semantic roles and communication protocols are specified in advance. We propose \textbf{NeuroMAS}, a framework that treats a multi-agent language system as a trainable and scalable neural-network-like architecture with hidden textual intermediate signals. In NeuroMAS, language-agent nodes are role-free but structure-aware: the topology determines how information can flow, while reward-based training determines how nodes communicate, specialize, and coordinate. This formulation shifts multi-agent design from workflow engineering toward architecture design, where depth, width, connectivity, and growth protocol become scalable sources of capability. We provide a theoretical perspective showing why such modular text-valued computation can be parameter-efficient when tasks admit stable decompositions. Experiments show that NeuroMAS improves over both inference-time collaboration methods and trained baselines, and that its gains are not explained by trainable parameter count alone. We further find that organizational scaling is path-dependent: larger systems can fail when trained from scratch, but become stronger when grown progressively from smaller trained systems. These results suggest that learned multi-agent organization is a promising scaling axis for language-model reasoning and generation.
\end{abstract}

\section{Introduction}

Large language model (LLM)-based agents are increasingly capable in coding, reasoning, mathematics, writing, and tool-use tasks, as shown by recent systems for agentic coding, task automation, and mathematical problem solving~\citep{anthropic2025claudecode,openai2025codex,openclaw2026,deepmind2025deepthink}. Yet current LLM agents still struggle with long-horizon execution, complex reasoning, and domain-specific adaptation~\citep{yao2023react,liu2023agentbench,wang2023voyager}. A common response is to scale the underlying model by increasing parameter count, training data, training compute, or inference-time computation~\citep{kaplan2020scaling,brown2020language,hoffmann2022training,openai2023gpt4}. This strategy remains powerful but increasingly costly: empirical scaling laws suggest sublinear gains in compute and model size, so further improvement often requires disproportionately larger resources~\citep{kaplan2020scaling,hoffmann2022training,sardana2023beyond}. For frontier-scale backbones, brute-force scaling also raises training and inference costs~\citep{brown2020language,chowdhery2023palm,fedus2022switch}. These trends motivate a crucial question: rather than only scaling the model itself, can we scale the organization of models?

Multi-agent language systems provide a natural path toward this alternative scaling axis. Rather than relying on one monolithic generation process, multiple LLM agents can decompose a problem, explore different reasoning paths, exchange intermediate information, and aggregate partial solutions. Related inference-time reasoning methods already show that structured computation can improve reasoning, for example through chain-of-thought prompting, least-to-most decomposition, tree search, or debate-style interaction~\citep{wei2022chain,zhou2023least,yao2023tree,du2023improving}. LLM-based multi-agent systems (MAS) extend this idea by turning reasoning from single-agent generation into a distributed process among communicating agents~\citep{wu2023autogen,hong2024metagpt,qian2024chatdev}. Recent work further shows that system organization matters: the number of agents, communication pattern, and interaction topology can substantially affect performance~\citep{li2024moreagents,zhuge2024gptswarm,zhou2025mass}. Thus, capability may come not only from the scale of an individual backbone model, but also from how multiple model instances are organized.

However, most existing multi-agent systems still treat organization primarily as a design choice rather than a scalable object. Prompt- and workflow-based systems assign agents human-written roles, such as planner, solver, critic, verifier, or reviewer, and specify communication protocols by hand~\citep{wu2023autogen,hong2024metagpt,qian2024chatdev,wang2025mixture}. More recent methods optimize parts of this design space, including prompts, communication edges, workflows, or topologies~\citep{zhuge2024gptswarm,zhou2025mass,yang2025agentnet,motwani2024malt}. These methods show that organization matters, but they usually separate organizational design from agent training: workflows are often fixed by the designer, while topology or prompt search is performed apart from end-to-end optimization. This leaves a central gap: we lack a unified formulation in which multi-agent organization and trainable agent parameters are optimized as one system.

\begin{figure}[t]
    \centering
    \includegraphics[width=0.99\textwidth]{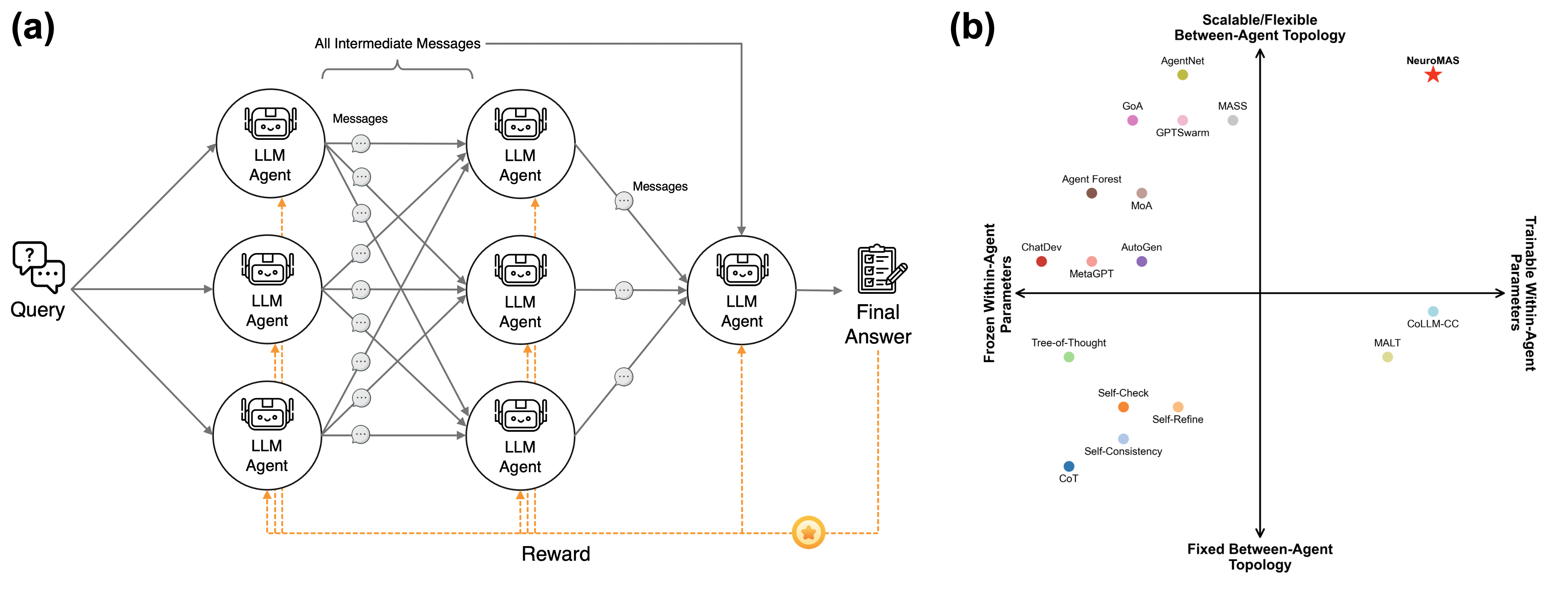}
    \caption{
    (a) Overview of the NeuroMAS Method. (b) Conceptual landscape of LLM multi-agent systems. The x-axis denotes within-agent trainability, while the y-axis denotes between-agent topology scalability. Positions are illustrative rather than quantitative.
    }
    \label{fig:neuromas}
    \vspace{-14pt}
\end{figure}
We address this gap with \textbf{Neural Multi-Agent System (NeuroMAS)}, a framework that treats a multi-agent language system as a trainable and scalable architecture with hidden textual intermediate signals. As shown in Figure~\ref{fig:neuromas}(a), NeuroMAS represents the system as language-agent nodes connected by text-carrying edges. Each node receives the task input and upstream messages, generates text for downstream nodes, and is optimized jointly with the full architecture using only the final task reward. Unlike conventional workflows, NeuroMAS does not assign hand-written roles such as planner, critic, verifier, or judge. Instead, nodes are \textbf{role-free but structure-aware}: each node knows its architectural position and output format, while its functional specialization is learned rather than prescribed. The topology defines information flow; training determines how nodes use incoming messages, communicate, and coordinate to produce the final answer.

This formulation reframes multi-agent design as \textbf{scalable and flexible neural topology design} rather than workflow engineering. Instead of hand-specifying roles, communication order, and aggregation rules, NeuroMAS treats the graph as an architecture: depth, width, connectivity, and growth protocol define the computation space, while learning determines effective module behavior. This view is inspired by neural architecture design, where structural choices such as depth, width, modularity, skip connections, graph computation, and architecture search strongly shape performance~\citep{lecun1998gradient,he2016deep,vaswani2017attention,shazeer2017outrageously,elsken2019neural}. NeuroMAS extends this perspective to multi-agent systems, aiming for communication, specialization, and coordination to emerge through reward-driven training. It also reflects a \textit{Bitter Lesson} view: durable progress should come from general trainable mechanisms that scale with computation, rather than increasingly elaborate human-authored workflows~\citep{sutton2019bitter}.

We evaluate NeuroMAS in a resource-constrained setting where agents share a frozen backbone and only lightweight node-specific adapters are trained. Across reasoning, domain-knowledge, and code-generation benchmarks, NeuroMAS outperforms fixed-backbone collaboration methods and trained single-model or multi-agent baselines. Parameter-count controls show that the gains are not merely due to more adapter parameters, while scaling experiments show that progressive growth is more reliable than training large topologies from scratch. These results support learned organization as a complementary scaling axis for language-model systems.

\noindent\textbf{Contributions.}
Our contributions are threefold. Methodologically, we introduce NeuroMAS, a trainable multi-agent language architecture with role-free, structure-aware nodes and text-valued communication.  {Theoretically, we analyze in what scenarios a modular MAS can be more parameter-efficient than a single LLM under hierarchical decompositions of the question.} Empirically, we show NeuroMAS's powerful performance across reasoning, domain-knowledge, and code-generation tasks.

\section{Related Work}
\label{sec:related_work}

\noindent\textbf{Structured inference with a single model.}
Many methods improve LLM reasoning by modifying test-time computation around a frozen model. Chain-of-thought exposes intermediate reasoning, self-consistency samples multiple reasoning paths and aggregates answers, and Tree-of-Thought searches over candidate trajectories~\citep{wei2022chain,wang2023selfconsistency,yao2023tree}. Self-Refine adds feedback-and-revision loops, while Self-Check verifies step-by-step reasoning without external tools or supervision~\citep{madaan2023selfrefine,miao2023selfcheck}. These methods show the value of structured inference, but the structure is manually specified and used only at inference time.

\noindent\textbf{Engineered multi-agent workflows.}
Another line of work distributes reasoning across multiple LLM agents while keeping model parameters fixed. AutoGen, MetaGPT, and ChatDev use explicit agent identities, role descriptions, hand-designed communication protocols, and software-style workflows~\citep{wu2023autogen,hong2024metagpt,qian2024chatdev}. Mixture-of-Agents and Agent Forest use multiple calls or agents to generate, compare, refine, aggregate, or vote over answers~\citep{wang2025mixture,li2024moreagents}. These systems show that agent organization matters, but roles, prompts, edges, and coordination rules are largely human-specified.

\noindent\textbf{Optimizing agent organization.}
Recent methods reduce manual design by optimizing parts of the multi-agent organization. GPTSwarm optimizes agent graphs, including node prompts and graph connectivity~\citep{zhuge2024gptswarm}. MASS searches over prompts and topologies through staged optimization~\citep{zhou2025mass}. Graph of Agents builds input-dependent collaboration structures for long-context modeling, and AgentNet studies decentralized coordination with dynamically evolving connectivity~\citep{joo2025graph,yang2025agentnet}. These works treat topology and communication as design variables, but typically optimize prompts, graphs, routing rules, search procedures, or external controllers while keeping the underlying LLM agents frozen.

\noindent\textbf{Training multi-agent collaboration.}
The closest line of work trains multi-agent reasoning systems directly. MALT post-trains a sequential generator-verifier-refiner pipeline for mathematical and commonsense reasoning~\citep{motwani2024malt}. CoLLM-CC trains decentralized LLM collaboration with multi-agent actor-critic learning and a centralized critic for sparse or long-horizon rewards~\citep{liu2026learning}. NeuroMAS also trains collaboration, but uses role-free nodes and is not tied to a single fixed role-based pipeline.

Taken together, these lines of work define the conceptual landscape in Figure~\ref{fig:neuromas}(b). The horizontal axis separates frozen from trainable within-agent parameters, and the vertical axis separates fixed from scalable or flexible between-agent topology. The placements are illustrative rather than quantitative. NeuroMAS occupies the trainable-agent and scalable-topology region, whereas prior methods mainly engineer frozen-agent workflows, optimize external organizational structures, or train collaboration within fixed role-based scaffolds.

\section{Preliminaries}
\label{sec:preliminaries}

We briefly introduce the notation used to formulate NeuroMAS as a reward-trained language architecture. The key point is that final answers and intermediate messages are discrete text sequences, so the system is optimized at the trajectory level rather than by backpropagating through sampled tokens.

\noindent\textbf{Text policies and task objective.}
Let \(\mathcal V\) denote the token vocabulary, and let \(\mathcal T\) denote the set of all finite token sequences over \(\mathcal V\). We consider text-input, text-output tasks with observed examples
$
\{(X_i,Y_i)\}_{i=1}^n, (X_i,Y_i)\in \mathcal T\times \mathcal T,
$
where \(X_i\) is an input query and \(Y_i\) is the desired answer. A language model with parameters \(\theta\) defines a stochastic policy over output text,
$
\widehat Y\sim \pi_\theta(\cdot\mid X).
$
Under the standard autoregressive factorization,
$
\pi_\theta(\widehat Y\mid X)
=
\prod_{t=1}^{S}
\pi_\theta(\widehat Y_t\mid X,\widehat Y_{<t}),
$
where \(\widehat Y=(\widehat Y_1,\ldots,\widehat Y_S)\) is the generated sequence and \(\widehat Y_{<t}\) denotes its prefix before token \(t\).
A generated answer is evaluated by a task-specific reward
$
r(Y,\widehat Y)\in [0,1].
$
In the reasoning benchmarks considered in this paper, we use binary rewards based on answer correctness. For exact-match tasks,
\(r(Y,\widehat{Y}) = \mathbf{1}\{\widehat{Y}=Y\}\), after applying the task-specific answer canonicalization described in the experimental protocol. For code-generation tasks, the binary reward is defined by whether the generated solution passes the task unit tests.
For a single language policy, the empirical expected reward objective is
$
J_n(\theta)
=
\frac{1}{n}\sum_{i=1}^{n}
\mathbb E_{\widehat Y_i\sim \pi_\theta(\cdot\mid X_i)}
\left[
r(Y_i,\widehat Y_i)
\right].
$
The observed input-output pairs are treated as fixed, and the randomness comes from the sampled model outputs.

\noindent\textbf{Score-function optimization.}
Because generated text is discrete, the reward cannot be differentiated through the sampled tokens. We therefore use the score-function identity, also known as REINFORCE. For a fixed example \((X,Y)\), the gradient  is 
$
\nabla_\theta
\mathbb E_{\widehat Y\sim \pi_\theta(\cdot\mid X)}
\left[
r(Y,\widehat Y)
\right]
=
\mathbb E_{\widehat Y\sim \pi_\theta(\cdot\mid X)}
\left[
r(Y,\widehat Y)
\nabla_\theta \log \pi_\theta(\widehat Y\mid X)
\right].
$
Implementation details, including the gradient with baseline, the sampled surrogate loss, token-level log-probability computation, and baseline update rule, are provided in Appendix~\ref{app:policy_gradient_implementation}.

\section{Method}
\label{sec:method}

NeuroMAS treats a multi-agent language system as a trainable \emph{text-valued neural architecture}. Instead of generating an answer with a single language policy, NeuroMAS composes multiple stochastic language-agent nodes. Each hidden node receives an input context, generates text, and sends textual messages to downstream nodes. The output node aggregates the final messages and generates the answer. The analogy to a neural network is architectural rather than algebraic: information flows through layers of trainable modules, but the hidden states are discrete text sequences rather than continuous vectors.

The method has three defining features. First, nodes are \emph{role-free}: no node is manually assigned a semantic role such as planner, solver, critic, verifier, or judge. Second, nodes are \emph{structure-aware}: each node is informed of its position in the architecture and the message format required by the topology. Third, all node policies are trained from the same terminal reward on the final answer, so any useful specialization or coordination must arise because it improves task performance.

Let \(\tau=(n_1,\ldots,n_L)\) denote a layered topology with \(L\) hidden layers, where \(n_\ell\) is the number of nodes in hidden layer \(\ell\). The topology is followed by a single output aggregation node. Given an input query \(X\), NeuroMAS induces a conditional distribution over final answers,
$
\widehat Y\sim \pi_{\Theta}(\cdot\mid X;\tau),
$
where \(\Theta\) denotes all trainable node parameters. The degenerate topology \(\tau=\emptyset\) recovers a single language policy, which corresponds to a single-model reinforcement learning.

%---------------------------
\subsection{Text-Valued Layered Architecture}
\label{sec:method_architecture}

As illustrated in Figure~\ref{fig:neuromas}(a), NeuroMAS is a feed-forward layered architecture whose edges carry text. We write \(v_{\ell,j}\) for the \(j\)-th hidden node in layer \(\ell\), and \(v_{\mathrm{out}}\) for the output aggregation node.

Each node is a stochastic language policy. For a hidden node \(v_{\ell,j}\), let \(C_{\ell,j}\) denote its input context and let \(Z_{\ell,j}\) denote the generated text:
$
Z_{\ell,j}
\sim
\pi_{\theta_{\ell,j}}(\cdot\mid C_{\ell,j}),
\ell=1,\ldots,L,\quad j=1,\ldots,n_\ell.
$
The output node similarly generates the final answer
$
\widehat Y
\sim
\pi_{\theta_{\mathrm{out}}}(\cdot\mid C_{\mathrm{out}}).
$
Here \(\theta_{\ell,j}\) and \(\theta_{\mathrm{out}}\) are the trainable parameters of the corresponding node policies.

The forward pass proceeds layer by layer. A first-layer node has no incoming messages, so its input context is built from the original query \(X\), its structural position, and the topology \(\tau\). For a later hidden node \(v_{\ell,j}\), the input context contains the original query together with all messages sent to it from the previous layer:
$
\mathcal I_{\ell,j}
=
\left\{
M_{\ell-1,k\rightarrow \ell,j}
:
k=1,\ldots,n_{\ell-1}
\right\}, \ell=2,\ldots,L.
$
A fixed node prompt template converts this information into \(C_{\ell,j}\). The template provides the task input, the node's layer and position, the topology, and the incoming messages, but it does not assign a semantic role. 

After generating \(Z_{\ell,j}\), each non-output node sends messages to the next layer. The generated text follows a fixed structured format, and a deterministic parser extracts recipient-specific messages. For \(\ell<L\),
$
\sigma_{\ell,j}(Z_{\ell,j})
=
\left(
M_{\ell,j\rightarrow \ell+1,1},
\ldots,
M_{\ell,j\rightarrow \ell+1,n_{\ell+1}}
\right),
$
where \(M_{\ell,j\rightarrow \ell+1,k}\) is the message sent from node \(v_{\ell,j}\) to node \(v_{\ell+1,k}\). For the final hidden layer, the parser extracts the message sent to the output aggregation node:
$
M_{L,j\rightarrow \mathrm{out}}
=
\sigma^{\mathrm{out}}_{L,j}(Z_{L,j}).
$
The output node then forms its input context \(C_{\mathrm{out}}\) from the original query \(X\), the topology \(\tau\), and the final-layer messages
$
\{M_{L,j\rightarrow \mathrm{out}}\}_{j=1}^{n_L},
$
and generates the final answer \(\widehat Y\).
Thus, a NeuroMAS forward pass can be summarized as a stochastic text-valued composition:
{\small
\[
H_0=X,\qquad
H_\ell\sim \Phi_{\theta_\ell}(\cdot\mid X,H_{\ell-1},\tau),
\quad \ell=1,\ldots,L,\qquad
\widehat Y\sim \pi_{\theta_{\mathrm{out}}}(\cdot\mid X,H_L,\tau),
\]
}
where \(H_\ell=(Z_{\ell,1},\ldots,Z_{\ell,n_\ell})\) denotes the generated texts from all nodes in layer \(\ell\), and \(\Phi_{\theta_\ell}\) denotes the collection of node policies in that layer. This composition emphasizes that the final answer is produced by layered text-valued computation rather than by a single monolithic language policy.  {This formulation also makes precise the structural analogy to neural networks. The node outputs
\(Z_{\ell,j}\) collectively form a layer-wise hidden state
\(H_\ell=(Z_{\ell,1},\ldots,Z_{\ell,n_\ell})\), and the collection of node policies
\(\Phi_{\theta_\ell}\) acts as a stochastic, text-valued layer transformation from
\(H_{\ell-1}\) to \(H_\ell\). The topology \(\tau\) specifies the architecture through which these hidden textual states are composed. The analogy is therefore structural: NeuroMAS preserves the layered organization of neural networks, while replacing continuous activations with sampled natural-language messages.}
The prompt templates and parsers are fixed implementation components and are described in Appendix~\ref{app:method_details}.
The parser details, including the handling of malformed outputs, are provided in Appendix~\ref{app:message_parsing}.

%---------------------------
\noindent\textbf{Role-Free but Structure-Aware Nodes.}
\label{sec:method_rolefree}
NeuroMAS separates \emph{structural position} from \emph{semantic role}. Nodes are not assigned roles such as planner, critic, verifier, or judge; each node is specified only by its position in the topology, its incoming messages, and its output format. The topology controls what information a node receives and where its output goes, while training determines how the node uses that information. Thus, specialization is not hand-coded, but emerges from the interaction between architecture and terminal reward. Even when nodes share the same frozen backbone and start from equivalent adapter states, they can learn different behaviors because they occupy different structural positions, receive different contexts, and update separate adapter parameters.

%---------------------------
\subsection{Joint Training from Terminal Reward}
\label{sec:method_training}

NeuroMAS communicates through sampled text. Since intermediate messages are discrete, we do not backpropagate through message contents. Instead, we view every node as a stochastic policy and optimize the full architecture using the terminal reward. This extends the single-policy REINFORCE formulation in Section~\ref{sec:preliminaries} to a multi-agent computation trace.

For a fixed input query \(X\), let
$
\xi
=
\left\{
(C_{\ell,j},Z_{\ell,j}):\ell=1,\ldots,L,\ j=1,\ldots,n_\ell
\right\}
\cup
\{(C_{\mathrm{out}},\widehat Y)\}
$
denote the full sampled computation trace. Since each node input context is determined by the original query and upstream generated messages, the probability of the trace factorizes as
$
\pi_{\Theta}(\xi\mid X)
=
\left[
\prod_{\ell=1}^{L}
\prod_{j=1}^{n_\ell}
\pi_{\theta_{\ell,j}}(Z_{\ell,j}\mid C_{\ell,j})
\right]
\pi_{\theta_{\mathrm{out}}}(\widehat Y\mid C_{\mathrm{out}}).
$
The empirical expected reward objective is
{\small
\begin{equation}
\label{eq:neuromas_objective}
J_n(\Theta)
=
\frac{1}{n}
\sum_{i=1}^{n}
\mathbb{E}_{\xi_i\sim \pi_{\Theta}(\cdot\mid X_i)}
\left[
r(Y_i,\widehat Y_i)
\right],
\end{equation}
}
where
$
\Theta=\{\theta_{\ell,j}\}_{\ell,j}\cup\{\theta_{\mathrm{out}}\}
$
collects all trainable node parameters. Applying the score-function identity to the full trace gives REINFORCE-style policy gradients \citep{williams1992reinforce}.
For a sampled trajectory on \((X,Y)\), all trainable nodes receive the same terminal reward \(r(Y,\widehat Y)\), but each node updates only its own parameters according to the log-probability of its own sampled output. Specifically, the gradient contribution for hidden node \(v_{\ell,j}\) is
$
r(Y,\widehat Y)
\nabla_{\theta_{\ell,j}}
\log
\pi_{\theta_{\ell,j}}(Z_{\ell,j}\mid C_{\ell,j}),
$
and the output node is updated analogously by
$
r(Y,\widehat Y)
\nabla_{\theta_{\mathrm{out}}}
\log
\pi_{\theta_{\mathrm{out}}}(\widehat Y\mid C_{\mathrm{out}}).
$
Thus, the parameters of different nodes are updated separately, but they are optimized under a shared system-level objective and a shared terminal reward. There are no intermediate labels, role-specific rewards, or node-specific supervision signals.
The sampled loss used to implement these node-wise policy-gradient updates is described in Appendix~\ref{app:policy_gradient_implementation}.

This shared-reward training design is consistent with recent multi-agent reinforcement learning approaches for LLM collaboration. For example, \citet{liu2026learning} study centralized and decentralized value-sharing strategies for multi-agent LLM training, and find that sharing a common value signal across agents can be more effective in long-horizon or sparse-reward settings. NeuroMAS follows the same cooperative principle in a simpler form: multiple nodes are trained from a shared global reward rather than from separately engineered node-level rewards.

%---------------------------
\noindent \textbf{Node parameterization.} NeuroMAS is agnostic to the node parameterization: node policies may be fully fine-tuned or instantiated with parameter-efficient adapters. In the adapter-based implementation used in our experiments, each node has its own trainable LoRA adapter over a shared frozen backbone; details are provided in Appendix~\ref{app:parameterization_details}.
%---------------------------

\noindent\textbf{Progressive Growth.}
\label{sec:method_growth}
% We use progressive growth to expand a trained NeuroMAS topology into a larger one. Given a trained topology \(\tau\) and its expansion \(\tau'\), nodes with corresponding positions inherit their trainable parameters, while newly added nodes are initialized separately. The expanded system is then further trained with the same reward in Equation~\eqref{eq:neuromas_objective}. This preserves the learned behavior of the smaller topology while adding capacity for communication and specialization. Details are provided in Appendix~\ref{app:growth_schedule}.
We use progressive growth to expand a trained NeuroMAS topology into a larger one. Given a trained topology \(\tau\) and its expansion \(\tau'\), nodes with corresponding positions inherit their trainable parameters, while newly added nodes are initialized separately. When expanding width, inherited nodes are mapped to the same layer and position in the expanded topology, and newly created positions are initialized with new adapters. The expanded system is then further trained with the same reward in Equation~\eqref{eq:neuromas_objective}. This preserves the learned behavior of the smaller topology while adding capacity for communication and specialization. Details are provided in Appendix~\ref{app:growth_schedule}.

\section{Theoretical analysis from a compositional-efficiency perspective}
\label{sec:theory}

We give a theoretical perspective on why a trainable multi-agent architecture can be parameter-efficient. This is not a training guarantee and does not claim that every multi-agent system outperforms every single-agent system. It isolates a representational advantage: when a task admits a stable hierarchical decomposition, a modular architecture can allocate parameters to simpler local components and compose their outputs, whereas an unstructured single policy must represent the full solution process at once.
This view parallels classical results showing that neural networks approximate compositional function classes more efficiently when their architecture exposes the underlying hierarchy~\citep{bauer2019deep,fan2024factor}. NeuroMAS provides an analogous architecture for language generation: width supports different intermediate computations, depth supports composition and refinement, and the output node aggregates textual states.

\noindent\textbf{Fixed-input exact-match error.}
We use the notation from Section~\ref{sec:preliminaries}, but write \(\pi\) for a generic conditional text generator. Consider a fixed input \(X\in\mathcal T\) with unique desired answer \(Y^*(X)\in\mathcal T\). A generator \(\pi(\cdot\mid X)\) samples \(\widehat Y\), and we evaluate it by
$
\mathcal E_X(\pi)
=
\Pr_{\widehat Y\sim \pi(\cdot\mid X)}
\left(
\widehat Y\neq Y^*(X)
\right).
$
This fixed-input formulation gives a capacity statement in trainable parameter budget, not a sample-complexity or optimization theorem.
 {Thus, the rate has a minimax-like interpretation: it describes the best achievable approximation error for an architecture class on the hardest problems in the corresponding problem class, rather than the performance of a particular trained estimator.}
Let \(\Pi_{\mathrm{multi},q}\) be modular multi-agent generators with at most \(q\) trainable parameters, and \(\Pi_{\mathrm{single},q}\) be single-agent generators with the same budget. Define
{\small
\[
\varepsilon_{\mathrm{multi},q}
=
\inf_{\pi\in\Pi_{\mathrm{multi},q}}
\mathcal E_X(\pi),
\qquad
\varepsilon_{\mathrm{single},q}
=
\inf_{\pi\in\Pi_{\mathrm{single},q}}
\mathcal E_X(\pi).
\]
}
Thus, \(\mathcal E_X(\pi)\) is the error of one generator, while \(\varepsilon_{\mathrm{multi},q}\) and \(\varepsilon_{\mathrm{single},q}\) are the best errors achievable by the two architecture classes.
We then make the following assumptions.
\begin{assumption}[Hierarchical compositional solution]
\label{ass:compositional_solution}
The ideal generator \(\pi^*\) can be represented by composing \(K\) simpler generative components
$
\psi_1^*,\ldots,\psi_K^*,
$
arranged in a parallel or hierarchical computation graph. The class \(\Pi_{\mathrm{multi},q}\) contains architectures that can instantiate this decomposition with component budgets \(q_1,\ldots,q_K\) satisfying \(\sum_{i=1}^K q_i\leq q\).
\end{assumption}
 {This assumption specifies the structural regime in which a modular architecture is meaningful. It does not claim that all tasks decompose, but restricts the comparison to tasks whose solutions can be organized into reusable intermediate computations, such as subquestions, evidence generation, aggregation, checking, or refinement. In NeuroMAS, these components correspond to trainable text-valued nodes whose roles are not manually specified but can emerge from their graph positions. This assumption is natural for many reasoning, planning, mathematical, and coding tasks, but not for tasks requiring one indivisible global judgment.}
For component \(i\), let \(Z_i^*\) and \(\widehat Z_i\) denote the ideal and approximating intermediate outputs, so the component-level error is \(\Pr(\widehat Z_i\neq Z_i^*)\).
\begin{assumption}[Local learnability]
\label{ass:local_learnability}
For component \(i\), let \(\eta_i(q_i)\) be the smallest component-level exact-match error achievable with \(q_i\) trainable parameters. There exist constants \(C_{\mathrm{loc}}>0\) and \(a>0\) such that
$
\eta_i(q_i)
\leq
C_{\mathrm{loc}}q_i^{-1/a},
i=1,\ldots,K.
$
\end{assumption}
 {This assumption separates the difficulty of the full task from the difficulty of its local subproblems. It means that, once the task is decomposed, each component can be approximated with a favorable local rate using its allocated trainable parameters. The exponent \(a\) should therefore be interpreted as the effective complexity of the subproblems exposed by the modular architecture, rather than the complexity of the original end-to-end problem.}

\begin{assumption}[Stable composition]
\label{ass:stable_composition}
There exists \(C_{\mathrm{stab}}>0\) such that, if component \(i\) has error at most \(e_i\), then the composed generator has final exact-match error at most
$
C_{\mathrm{stab}}\sum_{i=1}^{K} e_i.
$
\end{assumption}
 {This assumption rules out decompositions in which small local mistakes are catastrophically amplified by later components. It is satisfied in regimes where the downstream aggregation or refinement step is robust to occasional imperfect intermediate messages, so that the final error grows at most proportionally to the sum of component errors. When the communication protocol is highly brittle, a modular text-valued architecture such as NeuroMAS may be less suitable unless the topology or aggregation mechanism can provide sufficient robustness.}

\begin{assumption}[Unstructured global difficulty]
\label{ass:global_difficulty}
For the single-agent class, there exist constants \(c_{\mathrm{single}}>0\) and \(b>a\) such that, for sufficiently large \(q\),
$
\varepsilon_{\mathrm{single},q}
\geq
c_{\mathrm{single}}q^{-1/b}.
$
\end{assumption}
This is a comparison assumption: without exposed compositional structure, the effective approximation exponent is worse.
 {The condition \(b>a\) characterizes the regime where decomposition genuinely simplifies the approximation problem. Here \(a\) measures the complexity of local subproblems, while \(b\) measures the effective complexity of the full problem. This parallels compositional approximation theory~\citep{bauer2019deep,fan2024factor} developed for high-dimensional functions to be decomposed into low-dimensional functions, except that here we do not discuss dimensionality but the difficulty of the full problem and subproblems. Thus, \(b\geq a\) is natural for compositional tasks, and the condition \(b>a\) only excludes the boundary case \(b=a\) where decomposition does not reduce effective approximation complexity.}
Now, we are ready to present the main theorem.

\begin{theorem}[Compositional efficiency]
\label{thm:hierarchical_composition_bound}
Under Assumptions~\ref{ass:compositional_solution}--\ref{ass:global_difficulty}, and a fixed \(K\), we have
{\small
\[
\varepsilon_{\mathrm{multi},q}
=
O\left(K^{1+1/a}q^{-1/a}\right),
\qquad
\varepsilon_{\mathrm{single},q}
=
\Omega\left(q^{-1/b}\right).
\]
}
Since \(b > a\), there exists \(q_0>0\) such that, for all \(q\geq q_0\), we have
$
\varepsilon_{\mathrm{multi},q}
<
\varepsilon_{\mathrm{single},q}.
$
\end{theorem}
The proof is given in Appendix~\ref{app:proof}.
Theorem~\ref{thm:hierarchical_composition_bound} shows that, for fixed \(K\), modular composition pays only a polynomial overhead in \(K\), while achieving the local rate \(q^{-1/a}\). Since \(b> a\), this rate is improved over the unstructured single-agent rate.

\noindent\textbf{Implications for NeuroMAS.}
The result suggests that NeuroMAS can be parameter-efficient when the task has useful compositional structure. Its topology can distribute and compose intermediate textual states without preassigned semantic roles. This is a capacity result, not an optimization guarantee: the advantage may vanish if the task, topology, or parameter allocation is poorly matched, if the task does not admit a stable decomposition, or if a single model learns the same decomposition internally. This motivates comparisons with single-agent RL and parameter-count-matched controls.  {This conclusion is conditional on the compositional regime above: it formalizes when exposing task decomposition can improve approximation rates, rather than giving a universal guarantee that multi-agent architectures outperform single-agent ones. We therefore view this section as an exploratory approximation rate for the emerging field of trainable LLM-based multi-agent systems.}

%  {
% Additional theoretical results for the NeuroMAS training algorithm are
% provided in Appendix~\ref{app:terminal_reward_credit}. There, we show that the
% shared final reward gives an unbiased node-wise policy-gradient estimator for
% the common system-level objective, and that the stochastic-gradient noise scales
% linearly with the number of trainable nodes. These results further illustrate the statistical credit signal
% and noise behavior of the NeuroMAS training algorithm.
% }
\section{Experiments}
\label{sec:experiments}

We evaluate NeuroMAS across reasoning, domain-knowledge, and code-generation benchmarks, compare it with fixed-backbone inference and adapter-trained baselines, test backbone robustness, and study progressive growth. We first introduce the experimental setup below.

\noindent\textbf{Backbone models and training.}
We use three frozen backbone models: Qwen3-0.6B~\citep{yang2025qwen3}, Gemma-3-1B-IT~\citep{gemma2025gemma3} and DeepSeek-R1-Distill-Qwen-7B \citep{guo2025deepseek}. In all experiments, trainable parameters are introduced through LoRA adapters. For NeuroMAS, each node uses a LoRA adapter with rank \(r=8\). For the Single-LLM RL control, we use a rank-adjusted LoRA adapter so that its total number of trainable parameters matches NeuroMAS-3. Implementation details are provided in Appendix~\ref{app:exp-details}.

\textbf{Benchmark tasks and baseline methods.}
We denote by \(\text{NeuroMAS-}c\) a topology with \(c\) total LLM calls per input, including the output aggregation node. NeuroMAS-3, NeuroMAS-5, and NeuroMAS-7 correspond to progressively larger fully connected feedforward agent topologies. We evaluate on ARC-Challenge for science question answering~\citep{clark2018arc}, BBH-Navigate for spatial navigation reasoning~\citep{suzgun2022bbh}, three MMLU subjects covering mathematical reasoning, physical reasoning, and biomedical knowledge~\citep{hendrycks2021mmlu}, and HumanEval for code generation~\citep{chen2021humaneval}. The baselines include inference-only and trained methods. Full benchmark details and baseline descriptions are provided in Appendix~\ref{app:exp-details}.

\begin{table*}[t]
\centering
\fontsize{7.7pt}{9pt}\selectfont
\setlength{\tabcolsep}{3.0pt}
\caption{
Accuracy and computational cost across benchmarks of methods using Qwen3-0.6B.
}
\vspace{-2pt}
\label{tab:main}
\begin{tabular}{llcccccccc}
\toprule
& & \multicolumn{6}{c}{Accuracy (\%)} & \multicolumn{2}{c}{Computation Cost} \\
\cmidrule(lr){3-8} \cmidrule(lr){9-10}
  & Method & ARC & Navigate & Algebra & Physics & Medicine & HumanEval & Trainable Params & LLM Calls \\
\midrule
\multirow{7}{*}{\shortstack[l]{Fixed\\Backbone\\Methods}}
& Direct prompting & 24.5 & 40.5 & 22.0 & 17.6 & 24.0 & 11.0 & 0 & 1 \\
& Self-Refine & 37.5 & 40.5 & 25.0 & 24.5 & 37.5 & 21.3 & 0 & 2 \\
& Self-Check & 38.5 & 44.5 & 26.0 & 22.5 & 32.5 & 24.4 & 0 & 2 \\
& MoA & 34.5 & 44.5 & 24.0 & 22.5 & 40.5 & 12.2 & 0 & 5 \\
& GoA & 39.5 & 40.5 & 28.0 & 24.5 & 34.0 & 20.7 & 0 & 11 \\
& GPTSwarm & 46.0 & 40.5 & 25.0 & 25.5 & 43.5 & 12.8 & 0 & 11 \\
& AgentNet & 30.5 & 40.5 & 35.0 & 25.5  & 39.0 & 18.3 & 0 & 3 \\
\midrule
\multirow{6}{*}{\shortstack[l]{Trained\\Backbone\\Methods}}
& MALT & 53.5 & 43.0 & 35.0 & 31.4 & 38.0 & 29.3 & 6.9M & 3 \\
& CoLLM-CC & 55.0 & 40.0 & 41.0 & 35.3 & 33.0 & 29.9 & 6.9M  & 3 \\
& Single-LLM RL & 46.0 & 42.0 & 29.0 & 24.5 & 43.0 & 17.7 & 6.9M & 1 \\
& NeuroMAS-3 & \textbf{56.5} & 45.5 & 39.0 & \textbf{44.1} & \textbf{48.0} & 30.5 & 6.9M & 3 \\
& NeuroMAS-5 & 54.0 & 48.0 & 39.0 & 39.2 & 41.5 & 29.9 & 11.5M & 5 \\
& NeuroMAS-7 & 53.5 & \textbf{51.0} & \textbf{42.0} & 39.2 & 43.5 & \textbf{31.7} & 16.1M & 7 \\
\bottomrule
\end{tabular}
\vspace{-12pt}
\end{table*}
\noindent\textbf{Main Results.}
\label{sec:exp_main}
Table~\ref{tab:main} reports accuracy and computational cost across reasoning, domain knowledge, and code-generation benchmarks. The key comparison is NeuroMAS-3 against Single-LLM RL: both use the same frozen backbone, the same REINFORCE objective, and the same trainable-parameter budget, but NeuroMAS-3 distributes the adapter budget across a learned multi-agent topology.
NeuroMAS-3 outperforms the strongest fixed-backbone baseline on every benchmark and also improves over the parameter-matched Single-LLM RL control. Since the training objective, backbone, and adapter budget are controlled in this comparison, the gains support the value of learned multi-agent organization beyond reinforcement learning on a single agent. Across the full table, a NeuroMAS variant obtains the best result on every benchmark, with NeuroMAS-3 strongest as a compact topology and NeuroMAS-7 strongest on Navigate, Algebra, and HumanEval.
 {
We further compare NeuroMAS with equal-budget repeated sampling from the parameter-matched Single-LLM RL baseline, using the same number of inference-time calls and per-call token budget. Across $K\in\{3,5,7\}$ on Algebra, NeuroMAS consistently outperforms these repeated-sampling baselines; details are provided in Appendix~\ref{app:equal-budget}.
}
\noindent\textbf{Backbone Robustness.}
\label{sec:backbone_robustness}
To test whether the gains depend on Qwen3-0.6B, we repeat the ARC-Challenge evaluation using Gemma-3-1B-IT~\citep{gemma2025gemma3}. The full results are reported in Appendix Table~\ref{tab:main_gemma}. NeuroMAS-3 achieves the best result among all methods under the Gemma backbone, outperforming both fixed-backbone baselines and adapter-trained baselines. Because Single-LLM RL is again parameter-matched to NeuroMAS-3, this suggests that the advantage of learned multi-agent organization is not specific to Qwen3-0.6B. The Gemma results also show that adding more agents does not guarantee better performance: NeuroMAS-5 remains close to NeuroMAS-3, while NeuroMAS-7 degrades. This also motivates the progressive-growth ablation below.
 {
\noindent\textbf{7B-Scale Evaluation.}
Further, we conduct a larger-scale experiment using DeepSeek-R1-Distill-Qwen-7B. We evaluate on HMMT Feb 2026 \citep{dekoninck2026matharena}, a recent competition-math benchmark released after the model's training cutoff, which provides a fresh and less saturated setting. As shown in Appendix Table~\ref{tab:hmmt_r17b}, NeuroMAS-3 achieves the best accuracy, 15.2\%, compared with at most 9.1\% for other baselines. This suggests that learned multi-agent organization remains beneficial beyond the smaller 0.6B/1B backbones used in the main experiments. Details are provided in Appendix~\ref{app:hmmt-7b}.
}

%  {
% \begin{table}[t]
% \centering\small
% \caption{HMMT Feb 2026 results using DeepSeek-R1-Distill-Qwen-7B as the frozen backbone. NeuroMAS-3 achieves the best held-out accuracy under the unified evaluation protocol.}
% \label{tab:hmmt_7b}
% \begin{tabular}{lcc}
% \toprule
% Method & Type & HMMT accuracy \\
% \midrule
% NeuroMAS-3 (ours) & trained & \textbf{15.2\% (5/33)} \\
% Direct inference & inference & 9.1\% (3/33) \\
% Single-LLM RL & RL & 9.1\% (3/33) \\
% MALT & RL & 9.1\% (3/33) \\
% CoLLM-CC & RL & 9.1\% (3/33) \\
% MoA & inference & 9.1\% (3/33) \\
% GoA & inference & 9.1\% (3/33) \\
% GPTSwarm & inference & 9.1\% (3/33) \\
% AgentNet & inference & 9.1\% (3/33) \\
% Self-Consistency & inference & 6.1\% (2/33) \\
% Self-Refine & inference & 6.1\% (2/33) \\
% Self-Check & inference & 6.1\% (2/33) \\
% \bottomrule
% \end{tabular}
% \end{table}
% }

%  {These measurements motivated our use of HMMT Feb 2026 for the 7B-scale evaluation. The goal is not to replace the original benchmarks, but to use a fresh and unsaturated benchmark where post-training gains can be meaningfully measured for a stronger reasoning backbone.}

\noindent\textbf{Ablation Study: Progressive Growth.}
\label{sec:exp_ablation_growth}
Larger NeuroMAS topologies can improve performance on some tasks, especially Navigate and Algebra, but the robustness study shows that larger topologies do not improve automatically. We therefore isolate the effect of the growth protocol on Navigate. Appendix Table~\ref{tab:growth} shows that from-scratch training degrades as the topology grows, whereas progressive growth improves monotonically on Navigate. This indicates that organizational scaling is path-dependent: larger systems work better when grown around a stable learned core.
 {
We further study a controlled growth step from NeuroMAS-3 to NeuroMAS-5. Warm-started variants converge faster than from-scratch training, while native placement achieves the best final accuracy. These results suggest that progressive growth mainly improves optimization through warm-starting, with learned communication conventions providing an additional benefit. Details are provided in Appendix~\ref{app:growth-mechanism}.
\noindent\textbf{Prompt Sensitivity.}
NeuroMAS uses a fixed node prompt template and routing parser during training, so we examine whether its performance is sensitive to these design choices. On ARC-Challenge, we re-trained the full progressive pipeline under three one-factor perturbations: removing the structural header, paraphrasing the routing instruction, and relaxing the routing parser, while keeping all other settings fixed. As shown in Appendix~\ref{app:prompt-parser-sensitivity}, these perturbations do not cause systematic degradation, suggesting that NeuroMAS is not strongly dependent on the exact template wording or parser format.
}
 {\noindent\textbf{Study of Role Specialization.}
Finally, we conduct a role-specialization study to test whether trained NeuroMAS nodes acquire differentiated functions without hand-written roles. The results show position-dependent node behaviors after training, suggesting that specialization emerges from learned adapters and communication topology. Details are provided in Appendix~\ref{app:role_specialization}.
}

\section{Discussion}

We introduced NeuroMAS, which treats a multi-agent language system as a trainable neural architecture with textual intermediate states: only the communication topology and each node's position are specified, while behaviors are learned from the terminal reward. This points to a complementary scaling direction: the organization around a backbone, not only the backbone itself.
NeuroMAS outperforms fixed-backbone and trained single- and multi-agent baselines across reasoning, domain-knowledge, code-generation, and 7B mathematical reasoning tasks. Parameter-matched comparisons with Single-LLM RL suggest the gains come from distributing training across modules, not extra adapters. Scaling is also path-dependent: large topologies can degrade when trained from scratch, whereas progressive growth expands smaller trained systems into stronger ones.
NeuroMAS is not a universal replacement for single-model scaling: it presumes decomposable tasks, suitable topology, and effective RL coordination, and multi-call inference raises cost. The compositional-efficiency result is an exploratory approximation-rate comparison, not a convergence guarantee. Future work includes richer rewards, adaptive topologies, tool use, long-horizon tasks, and larger-scale evaluation.

\noindent \textbf{Code Availability.} Code and tutorial for NeuroMAS are available at: \url{https://github.com/haoranlustat/NeuroMAS}

\begin{ack}
This work was partially supported by the U.S. National Science Foundation under grants  DMS-1903226, DMS-1925066, DMS-2124493, DMS-2311297, DMS-2319279, DMS-2318809, and by the U.S. National Institutes of Health under grant R01GM152814.
\end{ack}

\newpage

\bibliographystyle{plainnat}
\bibliography{references}

@article{williams1992reinforce,
  title     = {Simple Statistical Gradient-Following Algorithms for Connectionist Reinforcement Learning},
  author    = {Williams, Ronald J.},
  journal   = {Machine Learning},
  volume    = {8},
  number    = {3},
  pages     = {229--256},
  year      = {1992},
  publisher = {Springer},
  doi       = {10.1007/BF00992696}
}

@inproceedings{hu2022lora,
  title     = {{LoRA}: Low-Rank Adaptation of Large Language Models},
  author    = {Hu, Edward J. and Shen, Yelong and Wallis, Phillip and Allen-Zhu, Zeyuan and Li, Yuanzhi and Wang, Shean and Wang, Lu and Chen, Weizhu},
  booktitle = {International Conference on Learning Representations},
  year      = {2022}
}

@article{liu2026learning,
  title   = {Learning Decentralized {LLM} Collaboration with Multi-Agent Actor Critic},
  author  = {Liu, Shuo and Chen, Tianle and Amiri, Ryan and Amato, Christopher},
  journal = {arXiv preprint arXiv:2601.21972},
  year    = {2026}
}

@article{yang2025qwen3,
  title   = {{Qwen3} Technical Report},
  author  = {Yang, An and Li, Anfeng and Yang, Baosong and Zhang, Beichen and Hui, Binyuan and Zheng, Bo and others},
  journal = {arXiv preprint arXiv:2505.09388},
  year    = {2025}
}

@article{clark2018arc,
  title   = {Think You Have Solved Question Answering? Try {ARC}, the {AI2} Reasoning Challenge},
  author  = {Clark, Peter and Cowhey, Isaac and Etzioni, Oren and Khot, Tushar and Sabharwal, Ashish and Schoenick, Carissa and Tafjord, Oyvind},
  journal = {arXiv preprint arXiv:1803.05457},
  year    = {2018}
}

@inproceedings{suzgun2022bbh,
  title     = {Challenging {BIG-Bench} Tasks and Whether Chain-of-Thought Can Solve Them},
  author    = {Suzgun, Mirac and Scales, Nathan and Sch{\"a}rli, Nathanael and Gehrmann, Sebastian and Tay, Yi and Chung, Hyung Won and Chowdhery, Aakanksha and Le, Quoc V. and Chi, Ed H. and Zhou, Denny and Wei, Jason},
  booktitle = {Findings of the Association for Computational Linguistics: ACL 2023},
  pages     = {13003--13051},
  year      = {2023},
  publisher = {Association for Computational Linguistics}
}

@inproceedings{hendrycks2021mmlu,
  title     = {Measuring Massive Multitask Language Understanding},
  author    = {Hendrycks, Dan and Burns, Collin and Basart, Steven and Zou, Andy and Mazeika, Mantas and Song, Dawn and Steinhardt, Jacob},
  booktitle = {International Conference on Learning Representations},
  year      = {2021}
}

@article{kaplan2020scaling,
  title   = {Scaling Laws for Neural Language Models},
  author  = {Kaplan, Jared and McCandlish, Sam and Henighan, Tom and Brown, Tom B. and Chess, Benjamin and Child, Rewon and Gray, Scott and Radford, Alec and Wu, Jeffrey and Amodei, Dario},
  journal = {arXiv preprint arXiv:2001.08361},
  year    = {2020},
  doi     = {10.48550/arXiv.2001.08361}
}

@inproceedings{hoffmann2022training,
  title     = {Training Compute-Optimal Large Language Models},
  author    = {Hoffmann, Jordan and Borgeaud, Sebastian and Mensch, Arthur and Buchatskaya, Elena and Cai, Trevor and Rutherford, Eliza and de Las Casas, Diego and Hendricks, Lisa Anne and Welbl, Johannes and Clark, Aidan and Hennigan, Tom and Noland, Eric and Millican, Katie and van den Driessche, George and Damoc, Bogdan and Guy, Aurelia and Osindero, Simon and Simonyan, Karen and Elsen, Erich and Rae, Jack W. and Vinyals, Oriol and Sifre, Laurent},
  booktitle = {Advances in Neural Information Processing Systems},
  volume    = {35},
  pages     = {30016--30030},
  year      = {2022}
}

@misc{sutton2019bitter,
  title        = {{The Bitter Lesson}},
  author       = {Sutton, Richard S.},
  year         = {2019},
  howpublished = {\url{http://www.incompleteideas.net/IncIdeas/BitterLesson.html}},
  note         = {Online essay}
}

@inproceedings{madaan2023selfrefine,
  title     = {{Self-Refine}: Iterative Refinement with Self-Feedback},
  author    = {Madaan, Aman and Tandon, Niket and Gupta, Prakhar and Hallinan, Skyler and Gao, Luyu and Wiegreffe, Sarah and Alon, Uri and Dziri, Nouha and Prabhumoye, Shrimai and Yang, Yiming and Gupta, Shashank and Majumder, Bodhisattwa Prasad and Hermann, Katherine and Welleck, Sean and Yazdanbakhsh, Amir and Clark, Peter},
  booktitle = {Advances in Neural Information Processing Systems},
  volume    = {36},
  pages     = {46534--46562},
  year      = {2023}
}

@inproceedings{wu2023autogen,
  title     = {{AutoGen}: Enabling Next-Gen {LLM} Applications via Multi-Agent Conversations},
  author    = {Wu, Qingyun and Bansal, Gagan and Zhang, Jieyu and Wu, Yiran and Li, Beibin and Zhu, Erkang and Jiang, Li and Zhang, Xiaoyun and Zhang, Shaokun and Liu, Jiale and Awadallah, Ahmed Hassan and White, Ryen W. and Burger, Doug and Wang, Chi},
  booktitle = {Proceedings of the First Conference on Language Modeling},
  year      = {2024}
}

@inproceedings{qian2024chatdev,
  title     = {{ChatDev}: Communicative Agents for Software Development},
  author    = {Qian, Chen and Liu, Wei and Liu, Hongzhang and Chen, Nuo and Dang, Yufan and Li, Jiahao and Yang, Cheng and Chen, Weize and Su, Yusheng and Cong, Xin and Xu, Juyuan and Li, Dahai and Liu, Zhiyuan and Sun, Maosong},
  booktitle = {Proceedings of the 62nd Annual Meeting of the Association for Computational Linguistics (Volume 1: Long Papers)},
  pages     = {15174--15186},
  year      = {2024},
  address   = {Bangkok, Thailand},
  publisher = {Association for Computational Linguistics},
  doi       = {10.18653/v1/2024.acl-long.810}
}

@article{li2024moreagents,
  title   = {More Agents Is All You Need},
  author  = {Li, Junyou and Zhang, Qin and Yu, Yangbin and Fu, Qiang and Ye, Deheng},
  journal = {Transactions on Machine Learning Research},
  year    = {2024}
}

@inproceedings{wang2025mixture,
  title     = {Mixture-of-Agents Enhances Large Language Model Capabilities},
  author    = {Wang, Junlin and Wang, Jue and Athiwaratkun, Ben and Zhang, Ce and Zou, James},
  booktitle = {International Conference on Learning Representations},
  year      = {2025}
}

@inproceedings{zhuge2024gptswarm,
  title     = {{GPTSwarm}: Language Agents as Optimizable Graphs},
  author    = {Zhuge, Mingchen and Wang, Wenyi and Kirsch, Louis and Faccio, Francesco and Khizbullin, Dmitrii and Schmidhuber, J{\"u}rgen},
  booktitle = {Proceedings of the 41st International Conference on Machine Learning},
  series    = {Proceedings of Machine Learning Research},
  volume    = {235},
  pages     = {62743--62767},
  year      = {2024},
  publisher = {PMLR}
}

@inproceedings{zhou2025mass,
  title     = {Multi-Agent Design: Optimizing Agents with Better Prompts and Topologies},
  author    = {Zhou, Han and Wan, Xingchen and Sun, Ruoxi and Palangi, Hamid and Iqbal, Shariq and Vuli{\'c}, Ivan and Korhonen, Anna and Ar{\i}k, Sercan {\"O}.},
  booktitle = {International Conference on Learning Representations},
  year      = {2026}
}

@inproceedings{yang2025agentnet,
  title     = {{AgentNet}: Decentralized Evolutionary Coordination for {LLM}-Based Multi-Agent Systems},
  author    = {Yang, Yingxuan and Chai, Huacan and Shao, Shuai and Song, Yuanyi and Qi, Siyuan and Rui, Renting and Zhang, Weinan},
  booktitle = {Advances in Neural Information Processing Systems},
  year      = {2025}
}

@inproceedings{motwani2024malt,
  title     = {{MALT}: Improving Reasoning with Multi-Agent {LLM} Training},
  author    = {Motwani, Sumeet Ramesh and Smith, Chandler and Das, Rocktim Jyoti and Rafailov, Rafael and Torr, Philip H. S. and Laptev, Ivan and Pizzati, Fabio and Clark, Ronald and Schroeder de Witt, Christian},
  booktitle = {Proceedings of the Second Conference on Language Modeling},
  year      = {2025}
}

@inproceedings{shazeer2017outrageously,
  title     = {Outrageously Large Neural Networks: The Sparsely-Gated Mixture-of-Experts Layer},
  author    = {Shazeer, Noam and Mirhoseini, Azalia and Maziarz, Krzysztof and Davis, Andy and Le, Quoc V. and Hinton, Geoffrey E. and Dean, Jeff},
  booktitle = {International Conference on Learning Representations},
  year      = {2017}
}

@article{openai2023gpt4,
  title   = {{GPT-4} Technical Report},
  author  = {{OpenAI}},
  journal = {arXiv preprint arXiv:2303.08774},
  year    = {2023}
}

@inproceedings{wei2022chain,
  title     = {Chain-of-Thought Prompting Elicits Reasoning in Large Language Models},
  author    = {Wei, Jason and Wang, Xuezhi and Schuurmans, Dale and Bosma, Maarten and Ichter, Brian and Xia, Fei and Chi, Ed H. and Le, Quoc V. and Zhou, Denny},
  booktitle = {Advances in Neural Information Processing Systems},
  volume    = {35},
  pages     = {24824--24837},
  year      = {2022}
}

@inproceedings{yao2023react,
  title     = {{ReAct}: Synergizing Reasoning and Acting in Language Models},
  author    = {Yao, Shunyu and Zhao, Jeffrey and Yu, Dian and Du, Nan and Shafran, Izhak and Narasimhan, Karthik and Cao, Yuan},
  booktitle = {International Conference on Learning Representations},
  year      = {2023}
}

@inproceedings{yao2023tree,
  title     = {Tree of Thoughts: Deliberate Problem Solving with Large Language Models},
  author    = {Yao, Shunyu and Yu, Dian and Zhao, Jeffrey and Shafran, Izhak and Griffiths, Thomas L. and Cao, Yuan and Narasimhan, Karthik},
  booktitle = {Advances in Neural Information Processing Systems},
  volume    = {36},
  year      = {2023}
}

@inproceedings{hong2024metagpt,
  title     = {{MetaGPT}: Meta Programming for A Multi-Agent Collaborative Framework},
  author    = {Hong, Sirui and Zhuge, Mingchen and Chen, Jonathan and Zheng, Xiawu and Cheng, Yuheng and Wang, Jinlin and Zhang, Ceyao and Wang, Zili and Yau, Steven Ka Shing and Lin, Zijuan and Zhou, Liyang and Ran, Chenyu and Xiao, Lingfeng and Wu, Chenglin and Schmidhuber, J{\"u}rgen},
  booktitle = {International Conference on Learning Representations},
  year      = {2024}
}

@inproceedings{he2016deep,
  title     = {Deep Residual Learning for Image Recognition},
  author    = {He, Kaiming and Zhang, Xiangyu and Ren, Shaoqing and Sun, Jian},
  booktitle = {Proceedings of the IEEE Conference on Computer Vision and Pattern Recognition},
  pages     = {770--778},
  year      = {2016}
}

@article{elsken2019neural,
  title   = {Neural Architecture Search: A Survey},
  author  = {Elsken, Thomas and Metzen, Jan Hendrik and Hutter, Frank},
  journal = {Journal of Machine Learning Research},
  volume  = {20},
  number  = {55},
  pages   = {1--21},
  year    = {2019}
}

@inproceedings{brown2020language,
  title     = {Language Models are Few-Shot Learners},
  author    = {Brown, Tom B. and Mann, Benjamin and Ryder, Nick and Subbiah, Melanie and Kaplan, Jared and Dhariwal, Prafulla and Neelakantan, Arvind and Shyam, Pranav and Sastry, Girish and Askell, Amanda and others},
  booktitle = {Advances in Neural Information Processing Systems},
  volume    = {33},
  pages     = {1877--1901},
  year      = {2020}
}

@article{chowdhery2023palm,
  title   = {{PaLM}: Scaling Language Modeling with Pathways},
  author  = {Chowdhery, Aakanksha and Narang, Sharan and Devlin, Jacob and Bosma, Maarten and Mishra, Gaurav and Roberts, Adam and Barham, Paul and Chung, Hyung Won and Sutton, Charles and Gehrmann, Sebastian and others},
  journal = {Journal of Machine Learning Research},
  volume  = {24},
  number  = {240},
  pages   = {1--113},
  year    = {2023}
}

@article{fedus2022switch,
  title   = {Switch Transformers: Scaling to Trillion Parameter Models with Simple and Efficient Sparsity},
  author  = {Fedus, William and Zoph, Barret and Shazeer, Noam},
  journal = {Journal of Machine Learning Research},
  volume  = {23},
  number  = {120},
  pages   = {1--39},
  year    = {2022}
}

@inproceedings{sardana2023beyond,
  title     = {Beyond Chinchilla-Optimal: Accounting for Inference in Language Model Scaling Laws},
  author    = {Sardana, Nikhil and Portes, Jacob and Doubov, Sasha and Frankle, Jonathan},
  booktitle = {Proceedings of the 41st International Conference on Machine Learning},
  series    = {Proceedings of Machine Learning Research},
  volume    = {235},
  pages     = {43445--43460},
  year      = {2024},
  publisher = {PMLR}
}

@inproceedings{zhou2023least,
  title     = {Least-to-Most Prompting Enables Complex Reasoning in Large Language Models},
  author    = {Zhou, Denny and Sch{\"a}rli, Nathanael and Hou, Le and Wei, Jason and Scales, Nathan and Wang, Xuezhi and Schuurmans, Dale and Cui, Claire and Bousquet, Olivier and Le, Quoc V. and Chi, Ed H.},
  booktitle = {International Conference on Learning Representations},
  year      = {2023}
}

@inproceedings{du2023improving,
  title     = {Improving Factuality and Reasoning in Language Models through Multiagent Debate},
  author    = {Du, Yilun and Li, Shuang and Torralba, Antonio and Tenenbaum, Joshua B. and Mordatch, Igor},
  booktitle = {Proceedings of the 41st International Conference on Machine Learning},
  series    = {Proceedings of Machine Learning Research},
  volume    = {235},
  pages     = {11733--11763},
  year      = {2024},
  publisher = {PMLR}
}

@article{joo2025graph,
  title   = {Graph of Agents: Principled Long Context Modeling by Emergent Multi-Agent Collaboration},
  author  = {Joo, Taejong and Ishida, Shu and Sosnovik, Ivan and Lim, Bryan and Rezaei-Shoshtari, Sahand and Gaier, Adam and Giaquinto, Robert},
  journal = {arXiv preprint arXiv:2509.21848},
  year    = {2025}
}

@inproceedings{miao2023selfcheck,
  title     = {{SelfCheck}: Using {LLMs} to Zero-Shot Check Their Own Step-by-Step Reasoning},
  author    = {Miao, Ning and Teh, Yee Whye and Rainforth, Tom},
  booktitle = {International Conference on Learning Representations},
  year      = {2024}
}

@article{bauer2019deep,
  title   = {On Deep Learning as a Remedy for the Curse of Dimensionality in Nonparametric Regression},
  author  = {Bauer, Benedikt and Kohler, Michael},
  journal = {The Annals of Statistics},
  volume  = {47},
  number  = {4},
  pages   = {2261--2285},
  year    = {2019},
  doi     = {10.1214/18-AOS1747}
}

@misc{anthropic2025claudecode,
  title        = {Claude Code},
  author       = {{Anthropic}},
  year         = {2025},
  howpublished = {\url{https://www.anthropic.com/product/claude-code}},
  note         = {Accessed: 2026-05-09}
}

@misc{openai2025codex,
  title        = {Introducing Codex},
  author       = {{OpenAI}},
  year         = {2025},
  howpublished = {\url{https://openai.com/index/introducing-codex/}},
  note         = {Accessed: 2026-05-09}
}

@misc{openclaw2026,
  title        = {{OpenClaw}: Personal {AI} Assistant},
  author       = {{OpenClaw}},
  year         = {2026},
  howpublished = {\url{https://openclaw.ai/}},
  note         = {Accessed: 2026-05-09}
}

@misc{deepmind2025deepthink,
  title        = {Advanced Version of {Gemini} with Deep Think Officially Achieves Gold-Medal Standard at the International Mathematical Olympiad},
  author       = {{Google DeepMind}},
  year         = {2025},
  howpublished = {\url{https://deepmind.google/blog/advanced-version-of-gemini-with-deep-think-officially-achieves-gold-medal-standard-at-the-international-mathematical-olympiad/}},
  note         = {Accessed: 2026-05-09}
}

@article{liu2023agentbench,
  title   = {{AgentBench}: Evaluating {LLMs} as Agents},
  author  = {Liu, Xiao and Yu, Hao and Zhang, Hanchen and Xu, Yifan and Lei, Xuanyu and Lai, Hanyu and Gu, Yu and Ding, Hangliang and Men, Kaiwen and Yang, Kejuan and others},
  journal = {arXiv preprint arXiv:2308.03688},
  year    = {2023}
}

@article{wang2023voyager,
  title   = {{Voyager}: An Open-Ended Embodied Agent with Large Language Models},
  author  = {Wang, Guanzhi and Xie, Yuqi and Jiang, Yunfan and Mandlekar, Ajay and Xiao, Chaowei and Zhu, Yuke and Fan, Linxi and Anandkumar, Anima},
  journal = {arXiv preprint arXiv:2305.16291},
  year    = {2023}
}

@article{fan2024factor,
  title     = {Factor Augmented Sparse Throughput Deep {ReLU} Neural Networks for High Dimensional Regression},
  author    = {Fan, Jianqing and Gu, Yihong},
  journal   = {Journal of the American Statistical Association},
  volume    = {119},
  number    = {548},
  pages     = {2680--2694},
  year      = {2024},
  publisher = {Taylor \& Francis}
}

@inproceedings{vaswani2017attention,
  title     = {Attention Is All You Need},
  author    = {Vaswani, Ashish and Shazeer, Noam and Parmar, Niki and Uszkoreit, Jakob and Jones, Llion and Gomez, Aidan N. and Kaiser, {\L}ukasz and Polosukhin, Illia},
  booktitle = {Advances in Neural Information Processing Systems},
  volume    = {30},
  year      = {2017}
}

@article{lecun1998gradient,
  title     = {Gradient-Based Learning Applied to Document Recognition},
  author    = {LeCun, Yann and Bottou, L{\'e}on and Bengio, Yoshua and Haffner, Patrick},
  journal   = {Proceedings of the IEEE},
  volume    = {86},
  number    = {11},
  pages     = {2278--2324},
  year      = {1998},
  publisher = {IEEE}
}

@article{guo2025deepseek,
  title={Deepseek-r1: Incentivizing reasoning capability in llms via reinforcement learning},
  author={Guo, Daya and Yang, Dejian and Zhang, Haowei and Song, Junxiao and Wang, Peiyi and Zhu, Qihao and Xu, Runxin and Zhang, Ruoyu and Ma, Shirong and Bi, Xiao and others},
  journal={arXiv preprint arXiv:2501.12948},
  year={2025}
}

@article{chen2021humaneval,
  title   = {Evaluating Large Language Models Trained on Code},
  author  = {Chen, Mark and Tworek, Jerry and Jun, Heewoo and Yuan, Qiming and Pinto, Henrique Ponde de Oliveira and Kaplan, Jared and Edwards, Harri and Burda, Yuri and Joseph, Nicholas and Brockman, Greg and others},
  journal = {arXiv preprint arXiv:2107.03374},
  year    = {2021}
}

@article{fu2025deep,
  title={Deep think with confidence},
  author={Fu, Yichao and Wang, Xuewei and Tian, Yuandong and Zhao, Jiawei},
  journal={arXiv preprint arXiv:2508.15260},
  year={2025}
}

@article{zheng2023judging,
  title={Judging llm-as-a-judge with mt-bench and chatbot arena},
  author={Zheng, Lianmin and Chiang, Wei-Lin and Sheng, Ying and Zhuang, Siyuan and Wu, Zhanghao and Zhuang, Yonghao and Lin, Zi and Li, Zhuohan and Li, Dacheng and Xing, Eric and others},
  journal={Advances in neural information processing systems},
  volume={36},
  pages={46595--46623},
  year={2023}
}

@article{gemma2025gemma3,
  title   = {{Gemma 3} Technical Report},
  author  = {{Gemma Team} and Kamath, Aishwarya and Ferret, Johan and Pathak, Shreya and Vieillard, Nino and Merhej, Ramona and Perrin, Sarah and others},
  journal = {arXiv preprint arXiv:2503.19786},
  year    = {2025}
}

@inproceedings{wang2023selfconsistency,
  title     = {Self-Consistency Improves Chain of Thought Reasoning in Language Models},
  author    = {Wang, Xuezhi and Wei, Jason and Schuurmans, Dale and Le, Quoc V. and Chi, Ed H. and Narang, Sharan and Chowdhery, Aakanksha and Zhou, Denny},
  booktitle = {International Conference on Learning Representations},
  year      = {2023}
}

@article{dekoninck2026matharena,
      title={Beyond Benchmarks: MathArena as an Evaluation Platform for Mathematics with LLMs}, 
      author={Jasper Dekoninck and Nikola Jovanović and Tim Gehrunger and Kári Rögnvaldsson and Ivo Petrov and Chenhao Sun and Martin Vechev},
      year={2026},
      eprint={2605.00674},
      archivePrefix={arXiv},
      primaryClass={cs.CL},
      url={https://arxiv.org/abs/2605.00674}, 
}

\clearpage

\appendix
\setcounter{table}{0}
\renewcommand{\thetable}{S\arabic{table}}
\begin{center}
{\LARGE \textbf{Appendix for ``NeuroMAS: Multi-Agent Systems as Neural Networks with Joint Reinforcement Learning''}}
\end{center}

\section{Additional Method Details}
\label{app:method_details}

\subsection{Policy-Gradient Implementation}
\label{app:policy_gradient_implementation}

This section describes how the score-function objective in Sections~\ref{sec:preliminaries} and~\ref{sec:method_training} is implemented.

For a single language policy, the REINFORCE with baseline can be implemented by minimizing the sampled surrogate loss
\begin{equation}
\label{eq:sampled_loss_single}
\mathcal L(\theta)
=
-
\bigl(r(Y,\widehat Y)-b(X,Y)\bigr)
\log \pi_\theta(\widehat Y\mid X),
\end{equation}
where \(b(X,Y)\) is a baseline that does not depend on the sampled output \(\widehat Y\). The baseline reduces variance without changing the expectation of the gradient estimator.

For autoregressive language generation, the sequence log-probability is
\[
\log \pi_\theta(\widehat Y\mid X)
=
\sum_{t=1}^{S}
\log \pi_\theta(\widehat Y_t\mid X,\widehat Y_{<t}).
\]
In implementation, we use the mean token log-probability
\[
\bar \ell
=
\frac{1}{S}
\sum_{t=1}^{S}
\log \pi_\theta(\widehat Y_t\mid X,\widehat Y_{<t}),
\]
which reduces sensitivity to the generated sequence length.
This length-normalized form is used as a practical surrogate for the sequence-level REINFORCE objective when generated outputs have variable lengths.

For NeuroMAS, the same construction is applied node-wise. For a hidden node \(v_{\ell,j}\) with input context \(C_{\ell,j}\) and generated text \(Z_{\ell,j}\), define
\[
\bar \ell_{\ell,j}
=
\frac{1}{S_{\ell,j}}
\sum_{t=1}^{S_{\ell,j}}
\log
\pi_{\theta_{\ell,j}}
\left(
Z_{\ell,j,t}
\mid
C_{\ell,j}, Z_{\ell,j,<t}
\right).
\]
The node-level surrogate loss is
\[
\mathcal L_{\ell,j}
=
-
\bigl(r(Y,\widehat Y)-b_{\ell,j}\bigr)
\bar \ell_{\ell,j}.
\]
The output node is treated analogously:
\[
\mathcal L_{\mathrm{out}}
=
-
\bigl(r(Y,\widehat Y)-b_{\mathrm{out}}\bigr)
\bar \ell_{\mathrm{out}}.
\]
The total sampled loss for one trajectory is
\begin{equation}
\label{eq:sampled_loss}
\mathcal L
=
\sum_{\ell=1}^{L}
\sum_{j=1}^{n_\ell}
\mathcal L_{\ell,j}
+
\mathcal L_{\mathrm{out}}.
\end{equation}
Thus, all nodes are weighted by the same terminal reward, but each node contributes the log-probability of its own generated text.

In our experiments, each node baseline is maintained as an exponential moving average of observed rewards,
\[
b_v \leftarrow \rho b_v + (1-\rho) r(Y,\widehat Y),
\]
with \(\rho=0.9\). The resulting loss is optimized with AdamW. 

\subsection{Prompt Templates and Message Parsing}
\label{app:message_parsing}

This section provides the fixed prompt templates and parsing rules used by NeuroMAS. The templates are not learned. They provide the original query, structural information, and incoming messages, but do not assign semantic roles such as planner, solver, critic, verifier, or judge. In the experiments, the topology \(\tau\) denotes a layered fully connected feed-forward communication graph, analogous to the layer structure of an MLP. Nodes are organized into hidden layers, and every node in layer \(\ell\) sends a message to every node in layer \(\ell+1\). The width of each layer specifies the topology, such as \(2\text{-}2\) for two hidden layers with two nodes per layer.

\noindent\textbf{Hidden-node prompt.}
For a hidden node \(v_{\ell,j}\), the input context \(C_{\ell,j}\) is constructed from the original query \(X\), the topology \(\tau\), the node position \((\ell,j)\), and the messages received from all nodes in the previous layer. For first-layer nodes, the incoming-message field is omitted. The hidden-node prompt template is:
\begin{verbatim}
[Network: {n_layers}-layer {topology}, {total_nodes} nodes |
 Layer {layer_index}/{n_layers}, Position {position}/{layer_size}]

PROBLEM:
{question}

PREVIOUS:
[from L{prev_layer}P{prev_position_1}]: {incoming_message_1}
[from L{prev_layer}P{prev_position_2}]: {incoming_message_2}
...

Write a private message to each downstream node.
Use exactly this format:

TO #1: <message for downstream node 1>
TO #2: <message for downstream node 2>
...
\end{verbatim}

For first-layer nodes, the \texttt{PREVIOUS} block is omitted because there are no incoming hidden-layer messages. The number of \texttt{TO \#k} fields is determined by the number of nodes in the next layer.

\noindent\textbf{Output-node prompt.}
The output node receives the original query \(X\), the topology \(\tau\), and the messages from the final hidden layer. The output-node prompt template is:
\begin{verbatim}
[Output node | Network: {n_layers}-layer {topology},
 {total_nodes} hidden nodes]

PROBLEM:
{question}

PREVIOUS:
[from L{final_layer}P1]: {final_layer_message_1}
[from L{final_layer}P2]: {final_layer_message_2}
...

{task_specific_answer_footer}
\end{verbatim}

The task-specific answer footer enforces the required output format. For ARC-Challenge, MMLU-Abstract Algebra, MMLU-College Physics, and MMLU-Professional Medicine, the footer is:
\begin{verbatim}
Answer with A, B, C, or D.
Answer:
\end{verbatim}

For BBH-Navigate, the footer is:
\begin{verbatim}
Answer Yes or No.
Answer:
\end{verbatim}

For HumanEval, the footer is:
\begin{verbatim}
Provide the function body:
{function_prompt}
\end{verbatim}

\noindent\textbf{Summarizer prompt.}
When a separate summarizer is used, it receives the original query \(X\), the topology \(\tau\), and the outputs of the relevant hidden nodes. The summarizer prompt template is:
\begin{verbatim}
[Summarizer | Network: {n_layers}-layer {topology},
 {total_nodes} hidden nodes total]

PROBLEM:
{question}

ALL OUTPUTS:
[N1]: {node_1_output}
[N2]: {node_2_output}
...

{task_specific_answer_footer}
\end{verbatim}

\noindent\textbf{Message parsing.}
Each non-output node generates text \(Z_{\ell,j}\) in the structured format specified by the hidden-node prompt. A deterministic parser extracts recipient-specific messages:
\[
\sigma_{\ell,j}(Z_{\ell,j})
=
\left(
M_{\ell,j\rightarrow \ell+1,1},
\ldots,
M_{\ell,j\rightarrow \ell+1,n_{\ell+1}}
\right),
\qquad \ell<L.
\]
For final hidden-layer nodes, the parser extracts the message sent to the output node:
\[
M_{L,j\rightarrow \mathrm{out}}
=
\sigma_{L,j}^{\mathrm{out}}(Z_{L,j}).
\]

If a generated text cannot be parsed, the parser uses a deterministic fallback: the unparsed text is assigned to the first recipient, and empty messages are assigned to the remaining recipients.

\subsection{LoRA Adapter Details}
\label{app:parameterization_details}

The NeuroMAS formulation does not require a particular parameterization of the node policies. In principle, each node could be an independent language model. In our experiments, however, all nodes share the same frozen backbone model and differ only through node-specific LoRA adapters~\citep{hu2022lora}. This keeps training feasible while allowing nodes to specialize according to their structural positions.

Let \(\pi_0\) denote the shared frozen backbone. For each node \(v\), we attach a node-specific LoRA adapter with parameters \(\alpha_v\). The node policy is
\[
\pi_v(\cdot\mid C_v)
=
\pi_0(\cdot\mid C_v;\alpha_v),
\]
where the backbone parameters are fixed and only \(\alpha_v\) is trained. The output aggregation node also has its own adapter. Thus, a topology \(\tau=(n_1,\ldots,n_L)\) uses \(1+\sum_{\ell=1}^{L}n_\ell\) node-specific adapters and the same number of LLM calls per forward pass.

For each adapted pretrained weight matrix \(W\), node \(v\) uses
\[
W_v'
=
W+\Delta W_v,
\qquad
\Delta W_v
=
B_vA_v,
\]
where \(A_v\) and \(B_v\) are node-specific low-rank matrices. In our implementation, each node has its own adapter parameters, while the underlying backbone weights are shared and frozen across all nodes.

During progressive growth, inherited nodes copy their previously trained adapter parameters. Newly added nodes are initialized with inactive adapters by setting
\[
B_v=0,
\]
so that \(\Delta W_v=0\). Thus, a newly added node initially behaves like the shared frozen backbone under its input context and specializes only through subsequent reward-based training.

\subsection{Progressive Growth Schedule}
\label{app:growth_schedule}

The main progressive-growth schedule uses three topology sizes:
\[
\tau_3=[1,1],
\qquad
\tau_5=[2,2],
\qquad
\tau_7=[2,2,2],
\]
corresponding to NeuroMAS-3, NeuroMAS-5, and NeuroMAS-7, respectively. The transition from \(\tau_3\) to \(\tau_5\) expands width, and the transition from \(\tau_5\) to \(\tau_7\) expands depth.

When a trained topology \(\tau\) is expanded to a larger topology \(\tau'\), nodes that have corresponding positions in the expanded topology inherit their adapter parameters. Newly added nodes are initialized with inactive adapters as described in Section~\ref{app:parameterization_details}. The expanded system is then trained further with the same terminal-reward objective.

In our implementation, inherited nodes are reused when their local input-output interface remains compatible after expansion. If a node's position changes in a way that changes its role in the computation, such as moving from the final hidden layer to an intermediate hidden layer, its adapter is still transferred but continues adapting during subsequent training rather than being treated as fixed.

\subsection{Algorithmic Summary}
\label{app:algorithmic_summary}

Given input \(X\) and topology \(\tau\), a NeuroMAS forward pass proceeds as follows:
\begin{enumerate}
    \item Construct the input context \(C_{1,j}\) for each first-layer node using the query \(X\), the node's structural position, and the topology.
    \item Sample the generated text \(Z_{1,j}\) from each first-layer node policy.
    \item Parse each \(Z_{1,j}\) into recipient-specific messages for the next layer.
    \item Repeat context construction, generation, and parsing for all subsequent hidden layers.
    \item Construct the output-node context \(C_{\mathrm{out}}\) from \(X\), \(\tau\), and the final hidden-layer messages.
    \item Sample the final answer \(\widehat Y\) from the output node.
\end{enumerate}

For reward-based training, compute the terminal reward \(r(Y,\widehat Y)\), assign this same reward to all sampled node outputs, and update the node-specific trainable parameters using the sampled loss in Equation~\eqref{eq:sampled_loss}. For progressive growth, train the current topology, expand it by transferring inherited adapters and initializing new adapters, and then continue reward-based training under the expanded topology.

%================================================================

\section{Proofs and Additional Theoretical Results}
\label{app:proof}

This appendix provides details for the theoretical results in Section~\ref{sec:theory}. The main text treats a conditional generator abstractly as a distribution over output text. We first connect this notation to autoregressive language models and then prove the compositional parameter-efficiency comparison and its fixed-error implication.

\subsection{Autoregressive Generators and Component Errors}
\label{app:autoregressive_generator}

A conditional generator \(\pi(\cdot\mid X)\) defines a distribution over finite text sequences. When implemented by an autoregressive language model, the probability of a generated answer
\[
\widehat Y=(\widehat Y_1,\ldots,\widehat Y_{\widehat T})
\]
factorizes as
\[
\pi(\widehat Y\mid X)
=
\prod_{t=1}^{\widehat T}
\pi(\widehat Y_t\mid X,\widehat Y_{<t}).
\]
Here \(\widehat Y_{\widehat T}\) may be taken to include the terminal EOS token. The theoretical results do not require this factorization, but it connects the abstract generator notation used in Section~\ref{sec:theory} to standard autoregressive language generation.

The component-level exact-match error in Assumption~\ref{ass:local_learnability} is defined analogously to the final exact-match error. If component \(i\) has ideal output \(Z_i^*\), and an approximating component produces \(\widehat Z_i\), then its component error is
\[
\Pr(\widehat Z_i\neq Z_i^*).
\]
The object \(Z_i^*\) may be an intermediate message, partial answer, aggregation state, or refined state. The theorem only requires that local component errors can be controlled and that the final composition is stable with respect to those errors.

\subsection{Proof of Theorem~\ref{thm:hierarchical_composition_bound}}

\begin{proof}
By Assumption~\ref{ass:compositional_solution}, the ideal solution process can be represented by composing \(K\) generative components
\[
\psi_1^*,\ldots,\psi_K^*.
\]
The same assumption ensures that the modular class \(\Pi_{\mathrm{multi},q}\) contains architectures that can instantiate this \(K\)-component computation graph. Construct a modular multi-agent generator by assigning one trainable node to approximate each component. Allocate the total trainable parameter budget \(q\) approximately evenly across the \(K\) components, so that
\[
q_i\approx \frac{q}{K}.
\]
Integer rounding only changes constants and is suppressed in the asymptotic notation. We therefore write \(q_i=q/K\) in the calculation below.

By Assumption~\ref{ass:local_learnability}, component \(i\) can be approximated with error
\[
\eta_i(q_i)
\leq
C_{\mathrm{loc}}q_i^{-1/a}
=
C_{\mathrm{loc}}
\left(
\frac{q}{K}
\right)^{-1/a}
=
C_{\mathrm{loc}}K^{1/a}q^{-1/a}.
\]
Thus, every component can be approximated to error at most
\[
\epsilon
=
C_{\mathrm{loc}}K^{1/a}q^{-1/a}.
\]

By Assumption~\ref{ass:stable_composition}, if the component errors are at most \(\epsilon\), then the final exact-match error of the composed generator is at most
\[
C_{\mathrm{stab}}K\epsilon.
\]
Substituting the bound on \(\epsilon\) gives
\[
C_{\mathrm{stab}}K\epsilon
=
C_{\mathrm{stab}}K
\left(
C_{\mathrm{loc}}K^{1/a}q^{-1/a}
\right)
=
C_{\mathrm{stab}}C_{\mathrm{loc}}K^{1+1/a}q^{-1/a}.
\]
Let
\[
C_{\mathrm{multi}}=C_{\mathrm{stab}}C_{\mathrm{loc}}.
\]
The constructed modular generator therefore has a final exact-match error at most
\[
C_{\mathrm{multi}}K^{1+1/a}q^{-1/a}.
\]
Since \(\varepsilon_{\mathrm{multi},q}\) is the smallest achievable error over all modular multi-agent generators with at most \(q\) trainable parameters, it is no larger than the error of this constructed generator. Therefore,
\[
\varepsilon_{\mathrm{multi},q}
\leq
C_{\mathrm{multi}}K^{1+1/a}q^{-1/a}.
\]
Equivalently,
\[
\varepsilon_{\mathrm{multi},q}
=
O\left(K^{1+1/a}q^{-1/a}\right).
\]

By Assumption~\ref{ass:global_difficulty}, for the single-agent class,
\[
\varepsilon_{\mathrm{single},q}
\geq
c_{\mathrm{single}}q^{-1/b},
\]
and hence
\[
\varepsilon_{\mathrm{single},q}
=
\Omega\left(q^{-1/b}\right).
\]
Because \(b>a\), we have \(1/a>1/b\), so
\[
\frac{q^{-1/a}}{q^{-1/b}}
=
q^{-(1/a-1/b)}
\rightarrow 0
\qquad
\text{as } q\rightarrow\infty.
\]
Since \(K\) is fixed,
\[
\frac{
C_{\mathrm{multi}}K^{1+1/a}q^{-1/a}
}{
c_{\mathrm{single}}q^{-1/b}
}
=
\frac{C_{\mathrm{multi}}K^{1+1/a}}{c_{\mathrm{single}}}
q^{-(1/a-1/b)}
\rightarrow 0.
\]
Therefore, there exists \(q_0>0\) such that, for all \(q\geq q_0\),
\[
C_{\mathrm{multi}}K^{1+1/a}q^{-1/a}
<
c_{\mathrm{single}}q^{-1/b}.
\]
Combining the upper bound on \(\varepsilon_{\mathrm{multi},q}\) with the lower bound on \(\varepsilon_{\mathrm{single},q}\), we obtain
\[
\varepsilon_{\mathrm{multi},q}
<
\varepsilon_{\mathrm{single},q}
\]
for all \(q\geq q_0\). This proves the theorem.
\end{proof}

\subsection{Additional Theoretical Results on Fixed-Error Parameter Comparison}
\label{app:fixed_error_comparison}
The fixed-budget comparison in Theorem~\ref{thm:hierarchical_composition_bound} can also be expressed as a fixed-error parameter comparison. For a target exact-match error \(\delta>0\), define
\[
q_{\mathrm{multi}}(\delta)
=
\inf\left\{
q:
\varepsilon_{\mathrm{multi},q}\leq \delta
\right\},
\qquad
q_{\mathrm{single}}(\delta)
=
\inf\left\{
q:
\varepsilon_{\mathrm{single},q}\leq \delta
\right\}.
\]
From Theorem~\ref{thm:hierarchical_composition_bound}, it is sufficient for the modular generator to choose \(q\) such that
\[
C_{\mathrm{multi}}K^{1+1/a}q^{-1/a}
\leq
\delta.
\]
Solving for \(q\) gives
\[
q
\geq
C_{\mathrm{multi}}^aK^{1+a}\delta^{-a}.
\]
Therefore,
\[
q_{\mathrm{multi}}(\delta)
=
O\left(K^{1+a}\delta^{-a}\right).
\]

For the unstructured single-agent generator, Assumption~\ref{ass:global_difficulty} gives
\[
\varepsilon_{\mathrm{single},q}
\geq
c_{\mathrm{single}}q^{-1/b}.
\]
Thus, to reach error at most \(\delta\), it is necessary that
\[
c_{\mathrm{single}}q^{-1/b}
\leq
\delta.
\]
Solving for \(q\) gives
\[
q
\geq
c_{\mathrm{single}}^b\delta^{-b}.
\]
Hence,
\[
q_{\mathrm{single}}(\delta)
=
\Omega\left(\delta^{-b}\right).
\]
Since \(b>a\), \(\delta^{-b}\) grows faster than \(\delta^{-a}\) as \(\delta\to 0\). Therefore, for sufficiently small \(\delta\), the modular class requires fewer trainable parameters than the unstructured single-agent class, up to constants and the polynomial overhead in \(K\).

\section{Additional Experimental Setting Details}
\label{app:exp-details}

\subsection{Topology, Benchmark, and Baseline Details}
\label{app:topology-benchmark-baseline-details}

\paragraph{NeuroMAS topologies.}
We write \(\text{NeuroMAS-}c\) for a NeuroMAS topology that uses \(c\) total language-model calls per forward pass, including the output aggregation node. Equivalently, if the hidden topology is \(\tau=(n_1,\ldots,n_L)\), then \(c = 1+\sum_{\ell=1}^{L} n_\ell\). Thus, NeuroMAS-3 corresponds to \(\tau=[1,1]\), NeuroMAS-5 corresponds to \(\tau=[2,2]\), and NeuroMAS-7 corresponds to \(\tau=[2,2,2]\). Each NeuroMAS node has its own rank-\(8\) LoRA adapter, so the number of adapters equals the number of LLM calls.

\paragraph{LoRA parameterization.}
For all NeuroMAS variants, every node-specific adapter uses LoRA rank \(r=8\). For Single-LLM RL control, we use a rank-adjusted LoRA adapter to ensure the total number of trainable parameters matches that of NeuroMAS-3 on the same backbone. This ensures that the comparison between NeuroMAS-3 and Single-LLM RL controls for both the reinforcement-learning objective and the total adapter-parameter budget. The same protocol is used for Qwen3-0.6B and Gemma-3-1B-IT.

\paragraph{Benchmarks.}
We evaluate six tasks: ARC-Challenge for the answer of science questions~\citep{clark2018arc}, Navigate from the BIG-Bench Hard problem set for spatial navigation reasoning~\citep{suzgun2022bbh}, three Massive Multitask Language Understanding (MMLU) topics covering mathematical reasoning, physical reasoning, and biomedical knowledge~\citep{hendrycks2021mmlu}, and HumanEval for code generation~\citep{chen2021humaneval}. This suite tests reasoning, domain knowledge, and code generation in a regime where the frozen backbone is not saturated.

\paragraph{Baselines.}
We compare against fixed-backbone and adapter-trained inference methods. The fixed-backbone methods include direct prompting, Self-Refine~\citep{madaan2023selfrefine}, Self-Check~\citep{miao2023selfcheck}, MoA~\citep{wang2025mixture}, GoA~\citep{joo2025graph}, AgentNet~\citep{yang2025agentnet}, and GPTSwarm~\citep{zhuge2024gptswarm}. These methods test whether NeuroMAS improves beyond prompt-only collaboration and graph-based orchestration without parameter updates. The adapter-trained methods include Single-LLM RL, MALT~\citep{motwani2024malt}, CoLLM-CC~\citep{liu2026learning}, and NeuroMAS variants. Single-LLM RL uses the same REINFORCE objective as NeuroMAS but with a single agent, making it the key control for separating learned organization from reinforcement learning alone.

\subsection{Implementation Details}
\label{app:objective}

We optimize adapter-trained methods with AdamW with learning rate \(2\times 10^{-5}\), and gradient clipping at \(1.0\). Evaluation uses greedy decoding with \texttt{max\_new\_tokens}=200. 

\subsection{Datasets and Evaluation Protocol}
\label{app:datasets}

Checkpoint selection is performed on a held-out development split that is disjoint from the final evaluation. We save checkpoints periodically during training and report the checkpoint with the best development accuracy under greedy decoding. The final test accuracy is computed only after checkpoint selection.

Evaluation accuracy is measured after canonicalizing the outputs into the expected response format. Multiple-choice and yes/no tasks use canonicalized answer matching; for HumanEval, correctness is assessed by executing the generated code against the official unit tests.

%-----------

%-------------------
\subsection{Computation Resources}
\label{app:compute}

Experiments were conducted on NVIDIA RTX 5090 and NVIDIA H100 GPUs.

\section{Additional Experimental Results}
\label{app:exp-res}
 {
\subsection{Equal-Budget Single-Model Sampling Comparison}
\label{app:equal-budget}
}

 {
To test whether NeuroMAS's gains can be explained by using more inference-time calls, we compare it with repeated sampling from a parameter-matched Single-LLM RL baseline. The baseline uses the same frozen Qwen3-0.6B backbone, the same trainable LoRA parameter budget matched with NeuroMAS-3, and the same REINFORCE training recipe. At inference time, both methods use the same number of model calls ($K\in{3,5,7}$), and each call has the same 200-token upper limit.
}

 {
We evaluate three single-model readout rules. Self-consistency \citep{wang2023selfconsistency} takes a majority vote over (K) sampled answers, with ties broken by the highest-log-probability sample. Best-of-(K) by confidence \citep{fu2025deep} selects the sampled answer with the highest model log-probability. Repeated sampling with an LLM-as-a-judge \citep{zheng2023judging} generates (K-1) candidate answers and uses the final call to select among them, so the judge call is counted within the same total (K)-call budget.
}

\begin{table}[H]
\centering\small
\caption{Equal-budget comparison on MMLU Abstract Algebra. All methods use the same number of inference-time calls and the same 200-token upper limit per call.}
\label{tab:equal-budget}
\begin{tabular}{ccccc}
\toprule
Calls (K) & Self-Consistency & Best-of-(K) & Sampling + Judge & NeuroMAS-(K) \\
\midrule
3 & 37.0 & 33.0 & 32.0 & \textbf{39.0} \\
5 & 32.0 & 32.0 & 27.0 & \textbf{39.0} \\
7 & 32.0 & 32.0 & 27.0 & \textbf{42.0} \\
\bottomrule
\end{tabular}
\end{table}

 {
As shown in Table~\ref{tab:equal-budget}, repeated sampling from a single RL-trained model does not match NeuroMAS under the same inference budget. This suggests that NeuroMAS's gains come from learned multi-agent organization rather than simply from using more test-time samples.
}

\subsection{Backbone Robustness Result Table}
\label{app:backbone_robustness_results}
We show the result Table \ref{tab:main_gemma} for Backbone Robustness.
\begin{table}[H]
\caption{
Backbone robustness results on ARC-Challenge using Gemma-3-1B-IT as the backbone.
Fixed-backbone methods do not update model parameters, while adapter-trained methods optimize LoRA adapters during post-training.
}
\label{tab:main_gemma}
\vspace{6pt}
\centering
\begin{tabular}{lccc}
\toprule
Method & ARC (\%) & Trainable Params & LLM Calls \\
\midrule
\multicolumn{4}{l}{\textit{Fixed backbone models}} \\
Direct prompting & 31.0 & 0 & 1 \\
Self-Refine & 41.0 & 0 & 2 \\
Self-Check & 37.5 & 0 & 2 \\
MoA & 34.5 & 0 & 5 \\
GoA & 38.5 & 0 & 11 \\
GPTSwarm & 43.0 & 0 & 11 \\
AgentNet & 36.5 & 0 & 3 \\
\midrule
\multicolumn{4}{l}{\textit{Trained backbone methods}} \\
MALT & 35.0 & 4.5M & 3 \\
CoLLM-CC & 40.5 & 4.5M & 3 \\
Single-LLM RL & 34.5 & 4.5M & 1 \\
NeuroMAS-3 & \textbf{44.5} & 4.5M & 3 \\
NeuroMAS-5 & 44.0 & 7.5M & 5 \\
NeuroMAS-7 & 41.0 & 10.4M & 7 \\
\bottomrule
\end{tabular}
\end{table}

%-------------------
 {
\subsection{7B-Scale Experiment}
\label{app:hmmt-7b}
}

 {
We use DeepSeek-R1-Distill-Qwen-7B as the frozen backbone and train only LoRA adapters. The evaluation benchmark is HMMT Feb 2026, which contains 33 competition-math problems released after the model's training cutoff. We use this benchmark for the 7B-scale evaluation because it provides a fresh and unsaturated held-out setting, while several of the original benchmarks are already near saturation for the 7B reasoning backbone.
}

 {
All methods are evaluated under the same batched generation protocol: HuggingFace generation with left-padding, batch size 16, greedy decoding, and a maximum of 4096 new tokens. The larger generation budget is used to avoid truncating R1-style reasoning traces. For RL-based methods, LoRA adapters are trained with REINFORCE using a shared binary terminal reward, an EMA reward baseline with coefficient 0.9, AdamW with learning rate $2\times10^{-5}$, gradient clipping at 1.0, 100 training steps, and best-checkpoint selection by development accuracy.
}

 {
For the Single-LLM RL control, we re-trained a rank-adjusted LoRA baseline using the same 7B backbone and evaluation protocol. This baseline uses LoRA rank $r=24$, giving 15.14M trainable parameters. Although its development accuracy improves to 16.67\%, its held-out HMMT accuracy remains 9.1\%, matching Direct inference and remaining below NeuroMAS-3.
}

\begin{table}[h]
\caption{
Results on HMMT using DeepSeek-R1-Distill-Qwen-7B.
}
\label{tab:hmmt_r17b}
\vspace{6pt}
\centering
\resizebox{0.92\textwidth}{!}{
\begin{tabular}{llccc}
\toprule
Category & Method & HMMT (\%) & Trainable Params & LLM Calls \\
\midrule
\multirow{8}{*}{Fixed backbone methods}
& Direct prompting & 9.1 & 0 & 1 \\
& Self-Refine & 6.1 & 0 & 2 \\
& Self-Check & 6.1 & 0 & 2 \\
& Self-Consistency & 6.1 & 0 & 5 \\
& MoA & 9.1 & 0 & 5 \\
& GoA & 9.1 & 0 & 11 \\
& GPTSwarm & 9.1 & 0 & 11 \\
& AgentNet & 9.1 & 0 & 3 \\
\midrule
\multirow{4}{*}{Trained backbone methods}
& MALT & 9.1 & 15.1M & 3 \\
& CoLLM-CC & 9.1 & 15.1M & 3 \\
& Single-LLM RL & 9.1 & 15.1M & 1 \\
& NeuroMAS-3 & \textbf{15.2} & 15.1M & 3 \\
\bottomrule
\end{tabular}
}
\vspace{-10pt}
\end{table}

\begin{table}[t]
\centering\small
\caption{Benchmark choice for the 7B-scale evaluation. Direct accuracy of DeepSeek-R1-Distill-Qwen-7B is already high on several original benchmarks, leaving limited headroom for measuring post-training gains. HMMT Feb 2026 is recent, post-cutoff, and unsaturated.}
\vspace{6pt}
\label{tab:benchmark-choice-7b}
\begin{tabular}{lcc}
\toprule
Benchmark & Direct accuracy & Role in 7B study \\
\midrule
ARC-Challenge & 82.5\% & original benchmark; near saturation \\
BBH Navigate & 86.0\% & original benchmark; near saturation \\
MMLU Abstract Algebra & 86.2\% & original benchmark; near saturation \\
MMLU College Physics & 90.0\% & original benchmark; near saturation \\
\midrule
HMMT Feb 2026 & 9.1\% & selected; post-cutoff and unsaturated \\
\bottomrule
\end{tabular}
\end{table}

 {
For the 7B-scale evaluation, we use HMMT Feb 2026 as the main benchmark. As shown in Table~\ref{tab:benchmark-choice-7b}, DeepSeek-R1-Distill-Qwen-7B already achieves high direct accuracy on several of the original benchmarks, including ARC-Challenge, BBH Navigate, and the two MMLU subsets. These results leave limited headroom for measuring whether post-training or multi-agent training brings additional gains, since performance on these tasks is already close to saturation for this stronger reasoning backbone.
In contrast, HMMT Feb 2026 is both recent and substantially less saturated: the same direct-answer baseline achieves only 9.1\% accuracy. This makes it a more informative stress test for the 7B setting, because improvements cannot be explained merely by an already strong pretrained model solving the benchmark directly. Therefore, our goal is not to replace the original benchmarks, but to add a fresh, post-cutoff, and challenging benchmark where the effect of NeuroMAS training can be measured more clearly.
}

\subsection{Progressive Growth Ablation Table}
\label{app:progressive_growth_results}

We show the results in Table \ref{tab:growth} for the progressive growth ablation study.

\begin{table}[t]
\centering
\normalsize
\caption{
Accuracy on the Navigate dataset under from-scratch and progressive-growth training.
}
\vspace{6pt}
\begin{tabular}{lcc}
\toprule
Method & From Scratch & Progressive Growth \\
\midrule
NeuroMAS-3 & 45.5 & 45.5 \\
NeuroMAS-5 & 41.0 & 48.0 \\
NeuroMAS-7 & 40.5 & 51.0 \\
\bottomrule
\end{tabular}
\label{tab:growth}
\end{table}

%---------
 {
\subsection{Ablation Study on the Mechanism of Progressive Growth}
\label{app:growth-mechanism}
}
 {
Progressive growth expands a trained smaller NeuroMAS topology into a larger topology by reusing trained node adapters and adding newly initialized nodes. This procedure may help for several reasons. First, it may stabilize optimization by giving the larger topology a better starting point. Second, it may preserve useful communication conventions learned in the smaller topology. Third, it may simply initialize part of the larger topology with better-trained adapters. We conduct a controlled growth-step ablation to separate these possible mechanisms.
}
 {
\paragraph{Setup.}
We isolate the NeuroMAS-3$\to$NeuroMAS-5 grow step, corresponding to NeuroMAS-3 growing into NeuroMAS-5. The trained smaller model contains one first-layer node $A$, one second-layer node $B$, and one summarizer $S$. In the trained NeuroMAS-3 topology, $A$ is the upstream node that produces an intermediate message, while $B$ is the downstream node that uses this message to produce an answer. The larger NeuroMAS-5 topology contains two first-layer positions, two second-layer positions, and the same summarizer position. Thus, when growing, the trained adapters $A$, $B$, and $S$ can be reused, while the two remaining worker-node positions are initialized randomly.
}
 {
\paragraph{Ablation design.}
We compare three warm-start placements and one from-scratch baseline, summarized in Table~\ref{tab:growth-placement}. In all warm-start variants, the reused old nodes remain \emph{trainable}; therefore, the ablation tests the effect of initialization rather than freezing old behaviors.
}
 {
\texttt{G\_standard} is the standard progressive-growth scheme: $A$ remains in the first layer, $B$ remains in the second layer, and $S$ remains the summarizer. \texttt{G\_swap} swaps the two trained worker-node layers, placing $B$ in the first layer and $A$ in the second layer. \texttt{G\_bothL2} places both trained worker nodes $A$ and $B$ in the two second-layer positions, leaving the entire first layer randomly initialized. \texttt{G\_scratch} trains the full NeuroMAS-5 topology from random initialization.
}
\begin{table}[h]
\centering
\small
\caption{Placement variants for the NeuroMAS-3$\to$NeuroMAS-5 progressive-growth ablation.}
\label{tab:growth-placement}
\begin{tabular}{lccc}
\toprule
Condition & First-layer positions & Second-layer positions & Summarizer \\
\midrule
\texttt{G\_standard} & $[A,\mathrm{new}]$ & $[B,\mathrm{new}]$ & $S$ \\
\texttt{G\_swap}     & $[B,\mathrm{new}]$ & $[A,\mathrm{new}]$ & $S$ \\
\texttt{G\_bothL2}   & $[\mathrm{new},\mathrm{new}]$ & $[A,B]$ & $S$ \\
\texttt{G\_scratch}  & $[\mathrm{new},\mathrm{new}]$ & $[\mathrm{new},\mathrm{new}]$ & $\mathrm{new}$ \\
\bottomrule
\end{tabular}
\end{table}

 {
This design distinguishes the three mechanisms. If progressive growth mainly works by preserving layer-specific communication conventions, then \texttt{G\_standard} should be substantially better than the mismatched warm-start placements. If it mainly works by reusing trained adapters, then all warm-started variants should outperform \texttt{G\_scratch}, even when the old nodes are placed in non-native positions. If it mainly stabilizes optimization, then warm-started variants should reach their performance plateau faster and show smoother development trajectories.
}

 {
We track the development accuracy trajectory and report final accuracy together with two optimization metrics: steps to reach 90\% of each run's own final accuracy, and detrended development-curve jitter. The experiment is run on the Navigate dataset. The results are reported in Table~\ref{tab:growth-results}.
}

\begin{table}[h]
\centering
\small
\caption{Optimization and final-accuracy results for the progressive-growth mechanism ablation on Navigate.}
\label{tab:growth-results}
\begin{tabular}{llccc}
\toprule
Condition & Initialization & Steps to 90\%$^\dagger$ & Dev std & Final acc. (\%) \\
\midrule
\texttt{G\_standard} & warm-start, native placement & \textbf{15} & \textbf{2.5} & \textbf{48} \\
\texttt{G\_swap}     & warm-start, swapped layers   & \textbf{15} & 3.3 & 46 \\
\texttt{G\_bothL2}   & warm-start, both in L2       & \textbf{15} & 3.1 & 47.5 \\
\texttt{G\_scratch}  & random initialization        & 30 & 3.1 & 41 \\
\bottomrule
\multicolumn{5}{l}{\footnotesize $^\dagger$ Development accuracy is measured every 15 steps.}
\end{tabular}
\end{table}

 {\paragraph{Results.} All warm-started variants reach 90\% of their own final accuracy by the first checkpoint, step 15, whereas training from scratch reaches this threshold at step 30. Thus, warm-starting gives an approximately $2\times$ speed-up in reaching the operating regime. The standard native-placement growth also has the smoothest trajectory, with the lowest detrended development-curve jitter, as shown in Table~\ref{tab:growth-results}. These results indicate that progressive growth improves the optimization process by giving the expanded topology a better starting point.
}

 {
The final accuracies further clarify the mechanism. \texttt{G\_standard} obtains the best final accuracy, 48\%, suggesting that preserving the original layer placement is beneficial. However, the two mismatched warm-start variants remain close: \texttt{G\_swap} reaches 46\%, and \texttt{G\_bothL2} reaches 47.5\%. Both outperform \texttt{G\_scratch}, which reaches 41\%. Therefore, the benefit of progressive growth is not explained solely by preserving the exact old communication convention. Even when trained nodes are inserted into non-native positions, continued training can adapt them to the larger topology, and the warm start still provides a better optimization state than random initialization.
}
 {
\paragraph{Interpretation.}
The ablation suggests that progressive growth mainly helps by accelerating and stabilizing optimization. Reusing trained adapters gives all warm-started variants an advantage over from-scratch training, while native placement gives the best and smoothest trajectory. This indicates that learned communication conventions are useful but not fragile: they can be partially relearned or repurposed when old nodes are placed in different positions. Thus, the main practical role of progressive growth is to provide a better initialization for the larger topology, while exact preservation of layer-specific roles is helpful but not necessary.
}

%----------------
 {
\subsection{Prompt Sensitivity}
\label{app:prompt-parser-sensitivity}
}

 {
Since the node prompt template and routing parser are used throughout RL training, they may affect learned node behavior. We therefore conducted a train-time sensitivity study on ARC-Challenge. Starting from the canonical NeuroMAS setup, we re-trained the full progressive pipeline NeuroMAS-3$\rightarrow$5$\rightarrow$7 under three one-factor perturbations: removing the structural header, paraphrasing the routing instruction, and relaxing the routing parser. All other factors were fixed, including the Qwen3-0.6B backbone, LoRA configuration, REINFORCE objective, progressive-growth schedule, data split, step budget, decoding rule, and random seed. We note that due to the computational time constraint, this study is set to half of the training steps of the main experiment in the main text.
}

 {
\begin{table}[h]
\centering\small
\caption{Prompt-template and parser sensitivity conditions. Each variant changes exactly one factor from the canonical setup.}
\label{tab:prompt-parser-conditions}
\begin{tabular}{@{}lll@{}}
\toprule
Condition & Factor & Change from canonical setup \\
\midrule
canonical & --- & Original structural header, routing instruction, and parser \\
no\_header & template & Remove the topology/layer/position header \\
paraphrase & template & Reword the routing instruction; keep the same structure \\
parse\_loose & parser & Relax the routing-parser regex \\
\bottomrule
\end{tabular}
\end{table}
}

 {
\begin{table}[h]
\centering\small
\caption{Train-time sensitivity on ARC-Challenge. Each cell reports accuracy (\%) on 120 evaluation examples.}
\label{tab:prompt-parser-results}
\begin{tabular}{@{}llccc@{}}
\toprule
Condition & Seed & NeuroMAS-3 & NeuroMAS-5 & NeuroMAS-7 \\
\midrule
\multirow{2}{*}{canonical}
& 42 & 48.3 & 46.7 & 40.0 \\
& 43 & 43.3 & 48.3 & 44.2 \\
\addlinespace
\multirow{2}{*}{no\_header}
& 42 & 45.8 & 47.5 & 50.0 \\
& 43 & 49.2 & 53.3 & 55.0 \\
\addlinespace
\multirow{2}{*}{paraphrase}
& 42 & 46.7 & 48.3 & 40.0 \\
& 43 & 47.5 & 46.7 & 40.8 \\
\addlinespace
\multirow{2}{*}{parse\_loose}
& 42 & 45.8 & 49.2 & 48.3 \\
& 43 & 50.0 & 46.7 & 49.2 \\
\bottomrule
\end{tabular}
\end{table}
}
 {
Results show that the perturbations do not cause systematic degradation. Paraphrasing the routing instruction closely matches the canonical template, the relaxed parser remains within the same range, and removing the structural header does not reduce performance on average. These results suggest that NeuroMAS is robust to moderate changes in the fixed prompt template and parser format. We therefore view the structural prompt mainly as a practical interface for specifying topology and node position, rather than as a source of hand-coded semantic roles.
}

%=================================

 {
\subsection{Study of Role Specialization in Trained NeuroMAS}
\label{app:role_specialization}
}

 {
We add a role-specialization study to examine whether trained NeuroMAS nodes acquire differentiated functions. This question is important because NeuroMAS does not manually assign semantic expert roles. Each node is defined only by its graph position, input messages, and output format. Therefore, if nodes specialize after training, the specialization should arise from the learned adapters and the communication structure rather than from hand-written role prompts.
}

 {
\paragraph{Quantitative Analysis.}
The goal of the quantitative analysis is to test whether each node adapter is most active in the structural context where it was trained. We use a locally retrained NeuroMAS-5 model on ARC with a frozen Qwen3-0.6B backbone and one LoRA adapter for each node. Since an untrained LoRA adapter is functionally equivalent to the frozen backbone at initialization, changes in the next-token distribution after training reflect the learned effect of the node adapter.
}

 {
For each adapter, we compare the adapter-equipped model with the frozen base model on the same input:
\[
\mathrm{JSD}
\left(
p_{\mathrm{adapter}}(\cdot \mid x),
p_{\mathrm{base}}(\cdot \mid x)
\right).
\]
A larger Jensen--Shannon divergence indicates a stronger training-induced deviation from the base model.
}

 {
We evaluate each adapter under three structural input contexts: layer-1 input, layer-2 input, and summarizer input. If an adapter only learns a generic task improvement, it should deviate from the base model across all contexts. If it learns a position-specific function, its largest deviation should appear in the context matching its training position.
}

\begin{table}[h]
\centering
\small
\caption{
Role-context specialization on ARC with NeuroMAS-5.
Each value is the Jensen--Shannon divergence between a node adapter and the frozen base model. Columns indicate the structural input context used for evaluation.
}
\label{tab:role_context_arc}
\begin{tabular}{lccc}
\toprule
\textbf{Adapter tested}
& \textbf{Layer-1 input}
& \textbf{Layer-2 input}
& \textbf{Summarizer input} \\
\midrule
Layer-1 node adapter & \textbf{0.0374} & 0.0009 & 0.0002 \\
Layer-2 node adapter & 0.0003 & \textbf{0.0045} & 0.0007 \\
\bottomrule
\end{tabular}

\vspace{0.5em}
\end{table}

 {
Table~\ref{tab:role_context_arc} shows a clear context-dependent pattern. The layer-1 adapter has its largest deviation under layer-1 inputs and is nearly inactive elsewhere. The layer-2 adapter has its largest deviation under layer-2 inputs and much smaller deviations in the other contexts. 
These results support the role-specialization interpretation. Training does not simply produce generic adapter drift. Instead, each trained adapter is selectively activated in the structural context corresponding to its position in the communication graph.
}

 {
\paragraph{Qualitative Analysis.}
The qualitative analysis illustrates how this specialization appears during inference. We inspect representative communication traces from the same trained NeuroMAS-5 model. These traces are not used as the primary evidence; rather, they show examples where downstream nodes perform nontrivial operations on upstream messages.
}
 {
We select three correct-final ARC examples from the trace pool. They illustrate three mechanisms: correcting noisy upstream candidates, selecting the correct upstream branch, and resolving layer-2 disagreement through the summarizer. Raw routing placeholders and repeated template echoes are omitted when they do not contain substantive content.
}
\begin{table}[h]
\centering
\small
\caption{
Representative communication traces from trained NeuroMAS-5 on ARC. All examples have correct final answers. The table highlights three downstream mechanisms: arithmetic correction, selection of the correct upstream branch, and layer-2 disagreement followed by summarizer adjudication.
}
\label{tab:r2c2-three-traces}
\begin{tabularx}{\linewidth}{p{0.13\linewidth}p{0.20\linewidth}X}
\toprule
\textbf{Trace}
& \textbf{Pattern}
& \textbf{Communication content / answer} \\
\midrule

\textbf{h\_205}
&
Downstream arithmetic correction
&
\textbf{Question:}
A student mixed 25 grams of salt into 1,000 grams of water. What is the mass of the saltwater mixture?
A. 975 grams
B. 1,000 grams
C. 1,025 grams
D. 2,500 grams

\textbf{Layer-1 node 1:} ``975. 1,025. 1,000. 2,500. The answer is 975 grams.''

\textbf{Layer-1 node 2:} ``975. 1,025. 1,000. 2,500. The answer is 975 grams.''

\textbf{Layer-2 node 1:} ``C. The problem says the student mixed 25 grams of salt into 1,000 grams of water. So total mass is $25+1000=1025$ grams. So answer is C.''

\textbf{Layer-2 node 2:} ``C. The mass of the saltwater mixture is 1,025 grams.''

\textbf{Summarizer:} ``C. The mass of the saltwater mixture is 1,025 grams.'' \\

\midrule

\textbf{h\_234}
&
Selection of the correct upstream branch
&
\textbf{Question:}
A young ocean animal has two parents. Each of the parents has eight arms. How many arms does the young animal most likely have?
A. 2
B. 4
C. 8
D. 16

\textbf{Layer-1 node 1:} ``16. The parents each have 8 arms, so the young animal has 8 arms from each parent, totaling 16.''

\textbf{Layer-1 node 2:} ``The answer is C. 8. The young animal has 8 arms.''

\textbf{Layer-2 node 1:} ``C. The young animal has 8 arms.''

\textbf{Layer-2 node 2:} ``C. The young animal has 8 arms.''

\textbf{Summarizer:} ``C. The young animal has 8 arms.'' \\

\midrule

\textbf{h\_238}
&
Layer-2 disagreement and adjudication
&
\textbf{Question:}
Students have a rock sample that may contain the mineral hematite. They know that hematite has a red or rusty streak. How can the students determine the color of the streak in the rock sample?
A. use a rock with the known mineral to strike the sample
B. rub the sample on a piece of unglazed clay tile
C. use a file to break off a piece of the sample
D. scratch the sample with an iron nail

\textbf{Layer-1 node 1:} ``Use a rock with the known mineral to strike the sample. Rub the sample on a piece of unglazed clay tile. Use a file to break off a piece of the sample.''

\textbf{Layer-1 node 2:} ``Use a rock with the known mineral to strike the sample. Rub the sample on a piece of unglazed clay tile. Use a file to break off a piece of the sample.''

\textbf{Layer-2 node 1:} ``D.''

\textbf{Layer-2 node 2:} ``B.''

\textbf{Summarizer:} ``B.'' \\

\bottomrule
\end{tabularx}
\end{table}

 {
Table~\ref{tab:r2c2-three-traces} shows that downstream nodes are not merely passive repeaters. In h\_205, the layer-2 nodes correct the incorrect layer-1 mass answer. In h\_234, they follow the correct branch despite conflicting upstream answers. In h\_238, the layer-2 nodes disagree, and the summarizer selects the correct answer. Together with the quantitative probe, these examples show that trained NeuroMAS nodes can transform, filter, and adjudicate information according to their structural positions. We also acknowledge the limitation that NeuroMAS models with 3, 5, or 7 nodes in the paper are still relatively small systems. The current evidence mainly shows differentiated behavior across layers, including correcting, filtering, and adjudicating upstream messages. More complex forms of task decomposition and specialization may require larger-scale NeuroMAS systems with more nodes and richer communication structures, and this would be an important future work to study the role of specialization in larger NeuroMAS systems.
}

\end{document}